\documentclass[twoside,11pt]{article}

\usepackage{blindtext}

\usepackage{jmlr2e}

\usepackage{hyperref}
\usepackage{url}
\usepackage{multirow}
\usepackage{amsmath, amssymb, ntheorem, booktabs, color}
\usepackage{amsfonts, mathtools}
\usepackage{pdfpages}
\usepackage{algorithm}
\usepackage{rotating}
\usepackage{tabulary}
\usepackage{bm}
\usepackage{subcaption}
\usepackage{url}
\usepackage{latexsym}
\usepackage{adjustbox}

\newcommand{\argmax}{\operatornamewithlimits{argmax}}
\newcommand{\argmin}{\operatornamewithlimits{argmin}}

\allowdisplaybreaks

\newtheorem{thm}{Theorem}[section]
\newtheorem{prop}{Proposition}[section]
\newtheorem{lem}{Lemma}[section]

\newtheorem{expl}{Example}[section]

\newtheorem{assumption}{Assumption}[section]

\usepackage{lastpage}

\ShortHeadings{Muon Converges}{Hideaki Iiduka}
\firstpageno{1}

\begin{document}

\title{Muon Converges under Heavy-Tailed Noise: Nonconvex H\"{o}lder-Smooth Empirical Risk Minimization}

\author{\name Hideaki Iiduka \email iiduka@cs.meiji.ac.jp \\
       \addr Department of Computer Science\\
       Meiji University\\
       1-1-1 Higashimita, Tama-ku, Kawasaki-shi, Kanagawa 214-8571 Japa
       }

\editor{}

\maketitle

\begin{abstract}
Muon is a recently proposed optimizer that enforces orthogonality in parameter updates by projecting gradients onto the Stiefel manifold, leading to stable and efficient training in large-scale deep neural networks. Meanwhile, the previously reported results indicated that stochastic noise in practical machine learning may exhibit heavy-tailed behavior, violating the bounded-variance assumption. In this paper, we consider the problem of minimizing a nonconvex H\"{o}lder-smooth empirical risk that works well with the heavy-tailed stochastic noise. We then show that Muon converges to a stationary point of the empirical risk under the boundedness condition accounting for heavy-tailed stochastic noise. In addition, we show that Muon converges faster than mini-batch SGD. 
\end{abstract}


\section{Introduction}
\subsection{Background}
Empirical risk minimization (ERM) is a central issue in training deep neural networks (DNNs) on certain training datasets. ERM is an optimization problem for minimizing an empirical risk (ER) defined by the sum of loss functions corresponding to the training set. Since a loss function such as the cross entropy loss is nonconvex, we can consider ERM to be a nonconvex minimization problem. 

Mini-batch stochastic gradient descent (SGD) \citep{robb1951,zinkevich2003,nem2009,gha2012,gha2013,umeda2025increasing} is a simple and useful optimizer for finding appropriate parameters of the DNN in the sense of minimizing the ER. A standard assumption when analyzing mini-batch SGD is the smoothness of the ER, i.e., the Lipschitz continuity of the gradient of the ER, since almost all analyses of mini-batch SGD have been based on the descent lemma (see, e.g., \citep[Lemma 5.7]{beck2017} for the descent lemma).

Mini-batch SGD uses a stochastic gradient of the ER that is randomly chosen from the gradients of loss functions. Hence, a discrepancy arises between the stochastic gradient and the true gradient of the ER. We call such a discrepancy the stochastic noise. It is commonly assumed that the stochastic noise is bounded in the sense of the expectation of the squared norm; i.e., the variance of the stochastic gradient is bounded. This is because, in theory, the boundedness condition of the variance works well with the descent lemma that is satisfied under the condition of smoothness of the ER.

However, the numerical results in \citep{pmlr-v97-simsekli19a,pmlr-v139-garg21b,pmlr-v238-battash24a,ahn2024linear} indicated that stochastic noise may exhibit heavy-tailed behavior. Heavy-tailed noise refers to stochastic noise whose distribution allows large fluctuations with non-negligible probability due to its slowly decaying tails. Moreover, it was reported that the stochastic noise of SGD can be heavy-tailed \citep{NEURIPS2020_abd1c782,pmlr-v139-hodgkinson21a,liu2025nonconvex}. Accordingly, the bounded variance condition of the stochastic gradient would be unrealistic in practical machine-learning problems. 

\subsection{Motivation}
A standard condition \citep[Assumption 1]{DBLP:conf/nips/ZhangKVKRKS20} to analyze optimizers under heavy-tailed noise is that the stochastic noise is bounded in the sense of the expectation of the $\mathfrak{p}$-th power of the norm, where $\mathfrak{p} \in (1,2]$. In particular, we say that the stochastic noise is heavy-tailed when $\mathfrak{p} \in (1,2)$ \citep[Assumption 1]{DBLP:conf/nips/ZhangKVKRKS20}. We call the expectation of the $\mathfrak{p}$-th power stochastic noise norm the $\mathfrak{p}$-variance of the stochastic gradient (The precise mathematical formulation of the $\mathfrak{p}$-variance of the stochastic gradient is given in Assumption \ref{assum:1}(A2)(ii)). Under the boundedness of the $\mathfrak{p}$-variance of the stochastic gradient, mini-batch SGD and its variants have been analyzed in \citep{DBLP:conf/nips/ZhangKVKRKS20,DBLP:conf/nips/CutkoskyM21,nguyen2023improved,pmlr-v202-sadiev23a,pmlr-v235-liu24bo}. Meanwhile, \cite{fatkhullin2025can} and \cite{yamada_2026} showed that the H\"{o}lder smoothness \citep{Hoelder1882} that is weaker than the smoothness works well in both theory and practice with the heavy-tailed stochastic noise (The precise mathematical formulation of the H\"{o}lder smoothness is given in Assumption \ref{assum:1}(A1)). The motivation behind this work is thus to show that, under the H\"{o}lder smoothness of the ER and the boundedness of the $\mathfrak{p}$-variance of the stochastic gradient, mini-batch SGD converges.

Many optimizers have been presented to accelerate mini-batch SGD. For example, adaptive gradient methods such as Adam \citep{adam} and its variant AdamW \citep{loshchilov2018decoupled} have become the de facto standard in modern deep learning, owing to their fast convergence and strong empirical performance across a wide range of tasks. Subsequent work has explored richer preconditioning strategies, including methods such as Shampoo \citep{pmlr-v80-gupta18a}, which leverage matrix-valued statistics of gradients. More recently, the Muon (Momentum orthogonalized by Newton-Schulz) optimizer \citep{jordan2024muon} has been proposed as a new optimizer that performs updates based on orthogonalized gradients. Convergence analyses of the Muon optimizer have been presented in \citep{tang2026a,sato2025convergenceboundcriticalbatch,pethick2025training,pethick2025generalized,nagashima2026improvedconvergenceratesmuon} under the smoothness or $(L_0, L_1)$-smoothness of the ER. Meanwhile, we are interested in verifying that, under the H\"{o}lder smoothness of the ER and the boundedness of the $\mathfrak{p}$-variance of the stochastic gradient, Muon converges faster than mini-batch SGD. 

\subsection{Main results}
This paper considers the ERM under the H\"{o}lder smoothness of the ER and the boundedness of the $\mathfrak{p}$-variance of the stochastic gradient and provides useful properties of mini-batch gradient (Section \ref{sec:2}). Let $(\bm{W}_t) \subset \mathbb{R}^{m \times n}$ be the sequence generated by an optimizer with step size $\eta_t$ and batch size $b_t$ to minimize the ER $f$. $\nu \in (0,1]$ appears in the definition of the H\"{o}lder smoothness (see Assumption \ref{assum:1}(A1)), and $\mathfrak{p} \in (1,2]$ appears in the definition of the $\mathfrak{p}$-variance of the stochastic gradient (see Assumption \ref{assum:1}(A2)(ii)). The following summarizes convergence of the Muon optimizer when the momentum parameter $\beta$ is $0$ (Section \ref{sec:4}), compared with mini-batch SGD (Section \ref{sec:3}). Section \ref{sec:5} shows that Muon with $\beta \neq 0$ has the same results in Section \ref{sec:4}. 

\subsubsection{Descent property}
\label{sec:1.3.1}
Let $\eta_t$ be a diminishing step size converging to $0$ and let $b_t$ be a constant or an increasing batch size. Then, for all $\epsilon > 0$, there exists $s_0 \in \mathbb{N}$ such that, for all $t \geq s_0$, $\eta_t^\nu < \frac{2}{L}$, $O(\eta_t^{1 + \nu} + \eta_t^{1 + \nu} b_t^{1 - \mathfrak{p}}) < \epsilon$, and $O ( \eta_t^{1 + \nu} + \eta_t b_t^{(1 - \mathfrak{p})\mathfrak{p}^{-1}}) < \epsilon$, where $L > 0$ is the H\"{o}lder constant (see Lemma \ref{lem:sgd_descent} for the definition)
and $O$ is Landau's symbol. Mini-batch SGD satisfies the following inequality (Lemma \ref{lem:sgd_descent}(ii)) based on the generalized descent lemma (see \eqref{g_descent_lemma}) under the assumption of H\"{o}lder smoothness of the ER: if $1 + \nu \leq \mathfrak{p}$ holds, then, for all $t \geq s_0$, 
\begin{align*}
\text{[SGD] } 
\mathbb{E}_{\bm{\xi}_t} \left[f (\bm{W}_{t+1}) |\bm{\xi}_{[t-1]} \right] 
&\leq
f(\bm{W}_{t})
\underbrace{- \eta_t \left( 1 - \frac{L \eta_t^\nu}{2} \right) \| \nabla f (\bm{W}_t) \|_{\mathrm{F}}^2}_{< 0} 
+ 
\underbrace{O \left( \eta_t^{1 + \nu} 
+ \frac{\eta_t^{1 + \nu}}{b_t^{\mathfrak{p}-1}} \right)}_{\to 0 \text{ } (t + \infty)}\\
&<
f(\bm{W}_{t}) 
+ \epsilon,
\end{align*}
where $\mathbb{E}_{\bm{\xi}_t} [f (\bm{W}_{t+1}) |\bm{\xi}_{[t-1]}]$ is the expectation of $f(\bm{W}_{t+1})$ with respect to a random variable $\bm{\xi}_t$ conditioned on $\bm{\xi}_{[t-1]} = (\bm{\xi}_0, \cdots, \bm{\xi}_{t-1})$ and $\|\cdot\|_{\mathrm{F}}$ is the Frobenius norm. Meanwhile, the Muon optimizer satisfies the following inequality (Lemma \ref{lem:muon_without_descent}(ii)): for all $t \geq s_0$,
\begin{align*}
\text{[Muon] }
\mathbb{E}_{\bm{\xi}_t} \left[f (\bm{W}_{t+1}) |\bm{\xi}_{[t-1]} \right] 
&\leq
f(\bm{W}_{t})
\underbrace{- \eta_t \| \nabla f (\bm{W}_t) \|_{\mathrm{F}}}_{< 0}
+ 
\underbrace{O \left( \eta_t^{1 + \nu} 
+  
\frac{\eta_t}{b_t^{\frac{\mathfrak{p} -1}{\mathfrak{p}}}} \right)}_{\to 0 \text{ } (t \to + \infty)}\\
&<
f(\bm{W}_{t}) + \epsilon.
\end{align*}
The above inequalities imply that mini-batch SGD and Muon have a descent property in the sense that $\mathbb{E}_{\bm{\xi}_t} [f (\bm{W}_{t+1}) |\bm{\xi}_{[t-1]}] < f (\bm{W}_t) + \epsilon \approx f (\bm{W}_t)$. The main difference between the above two inequalities is the exponent of $\| \nabla f (\bm{W}_t)\|_{\mathrm{F}}$. This difference between the two optimizers arises from the definition of the search direction. While the search direction of mini-batch SGD uses the mini-batch gradient (see \eqref{mini_batch} and $\bm{D}_t^{\mathrm{SGD}}$ in \eqref{SGD}), the search direction of Muon uses the point on the Stiefel manifold $\mathrm{St}(n,m) \coloneqq \{ \bm{O} \in \mathbb{R}^{m \times n} \colon \bm{O}^\top \bm{O} = \bm{I}_n \}$ closest to the mini-batch gradient (see $\bm{D}_t^{\mathrm{Muon}}$ in \eqref{Muon} and \eqref{Muon_without_beta}). 

\subsubsection{Convergence}
\label{sec:1.3.2}
The above inequalities, together with the super martingale convergence theorem \citep[Proposition 8.2.10]{bert}, ensure that mini-batch SGD and Muon with $\eta_t$ satisfying $\sum_{t=0}^{+ \infty} \eta_t = + \infty$ satisfy 
\begin{align*}
\begin{drcases}
\text{[SGD] } \sum_{t=0}^{+ \infty} \left( \eta_t^{1 + \nu} 
+ \frac{\eta_t^{1 + \nu}}{b_t^{\mathfrak{p}-1}} \right) < + \infty \\
\text{[Muon] } \sum_{t=0}^{+ \infty} \left( \eta_t^{1 + \nu} 
+  
\frac{\eta_t}{b_t^{\frac{\mathfrak{p} -1}{\mathfrak{p}}}} \right) < + \infty
\end{drcases}
\Rightarrow 
\liminf_{t \to + \infty} \| \nabla f (\bm{W}_t) \|_{\mathrm{F}} = 0 \text{ a.s.},
\end{align*}
which, together with the descent properties, implies that mini-batch SGD and Muon converge to a stationary point of $f$ that corresponds to either a local minimizer or a saddle point (Theorems \ref{thm:sgd_convergence} and \ref{thm:muon_without_convergence}).

\subsubsection{Convergence rate}
\label{sec:1.3.3}
Sections \ref{sec:1.3.1} and \ref{sec:1.3.2} indicate that both mini-batch SGD and Muon converge to appropriate points almost surely. The difference between the two optimizers is reflected in the convergence rate, since the main difference between them in Section \ref{sec:1.3.1} is the exponent of $\| \nabla f (\bm{W}_t)\|_{\mathrm{F}}$. Under certain assumptions, mini-batch SGD (Theorems \ref{thm:sgd_convergence_rate_upper} and \ref{thm:sgd_convergence_rate_lower}) and Muon (Theorems \ref{thm:muon_without_convergence_rate_upper} and \ref{thm:muon_without_convergence_rate_lower}) have the following convergence rate: there exists $s \in \mathbb{N}$ such that, for all $T \geq s$, 
\begin{align*}
&\text{[SGD] }
\frac{1}{\sum_{t= s}^{T} \eta_t}
\sum_{t= s}^{T} \eta_t 
\mathbb{E} \left[ \| \nabla f (\bm{W}_t)\|_{\mathrm{F}}^2 \right]
= \Theta \left( \frac{1}{\sum_{t= s}^{T} \eta_t} \right)\\
&\text{[Muon] }
\frac{1}{\sum_{t= s}^{T} \eta_t}
\sum_{t= s}^{T} \eta_t 
\mathbb{E} \left[ \| \nabla f (\bm{W}_t)\|_{\mathrm{F}} \right]
= \Theta \left( \frac{1}{\sum_{t= s}^{T} \eta_t} \right),
\end{align*}
where $f(T) = \Theta (g(T))$ implies that there exist $c_1, c_2 > 0$ and $t \in \mathbb{N}$ such that, for all $T \geq t$, $c_1 g(T) \leq f(T) \leq c_2 g(T)$. When $\eta_t = \frac{1}{(t+1)^a}$ ($t \in \{0\} \cup \mathbb{N}$) is used, where $a \in (0,1)$ satisfies $(1+\nu)a > 1$ (e.g., $a > \frac{1}{2}$ when $\nu = 1$), we have that $\frac{T^{1-a}}{1-a} \leq \sum_{t=1}^{T} \frac{1}{t^a} \leq \frac{T^{1-a}}{1-a} + 1$, i.e., $\sum_{t=1}^{T} \frac{1}{t^a} = \Theta (T^{1-a})$. Hence, mini-batch SGD has a $\Theta (\frac{1}{T^{1-a}})$ rate of convergence in the sense of the mean of the total expectation of {\em the squared norm} $\| \nabla f (\bm{W}_t)\|_{\mathrm{F}}^2$, while Muon has a $\Theta (\frac{1}{T^{1-a}})$ rate of convergence in the sense of the mean of the total expectation of {\em the norm} $\| \nabla f (\bm{W}_t)\|_{\mathrm{F}}$. In particular, since we have 
\begin{align*}
\min_{t \in \{1,2,\cdots,T \}} \mathbb{E} \left[ \| \nabla f (\bm{W}_t)\|_{\mathrm{F}} \right]
= 
\begin{dcases}
O \left( \frac{1}{T^{\frac{1-a}{2}}}  \right) \text{ } &(\text{SGD})\\
O \left( \frac{1}{T^{1-a}}  \right) \text{ } &(\text{Muon}),
\end{dcases} 
\end{align*}
we can check that Muon converges faster than mini-batch SGD. 

\subsection*{Notation and definitions}
Here, we describe the notation and state some definitions. Let $\mathbb{N}$ be the set of natural numbers. Let $[N] \coloneqq \{1,2,\cdots,N\}$ and $[0:N] \coloneqq \{0,1,\cdots,N\}$ for $N \in \mathbb{N}$. Let $\mathbb{R}_+ \coloneqq \{ x \in \mathbb{R} \colon x \geq 0 \}$. Let $\mathbb{R}^{m \times n}$ be the set of $m \times n$ matrices with inner product $\bm{W}_1 \bullet \bm{W}_2 \coloneqq \mathrm{Tr}(\bm{W}_1^\top \bm{W}_2)$ ($\bm{W}_1, \bm{W}_2 \in \mathbb{R}^{m \times n}$) and the norm $\|\bm{W}\|_{\mathrm{F}} \coloneqq \sqrt{\bm{W} \bullet \bm{W}}$, where $\mathrm{Tr}(\bm{X})$ is the trace of $\bm{X}$. The dual norm $\|\bm{W}\|_{2,*}$ of the spectral norm $\|\bm{W}\|_{2} \coloneqq \max \{ \| \bm{W} \bm{x} \|_2 \colon \|\bm{x}\|_2 \leq 1 \}$ is defined by $\|\bm{W}\|_{2,*} \coloneqq \max \{ \bm{W} \bullet \bm{X} \colon \|\bm{X} \|_{2} \leq 1 \}$, where $\|\bm{x}\|_2$ is the Euclidean norm of $\bm{x} \in \mathbb{R}^n$. $\bm{O}_{m \times n}$ denotes the $m \times n$ zero matrix and $\bm{I}_n$ denotes the $n \times n$ identity matrix.

$\mathrm{P}(A)$ denotes the probability of event $A$. $\mathbb{E}_\xi [\bm{X}(\xi)]$ denotes the expectation of a random variable $\bm{X}(\xi)$ with respect to a random variable $\xi$. The variance of $\bm{X}(\xi)$ with respect to $\xi$ is defined by $\mathbb{V}_\xi [\bm{X}(\xi)] \coloneqq \mathbb{E}_\xi [\| \bm{X}(\xi) - \mathbb{E}_\xi [\bm{X}(\xi)] \|_{\mathrm{F}}^2]$. Let $\mathfrak{p} > 1$. The $\mathfrak{p}$-variance of $\bm{X}(\xi)$ with respect to $\xi$ is defined by $\mathbb{V}_\xi^{\mathfrak{p}} [\bm{X}(\xi)] \coloneqq \mathbb{E}_\xi [\| \bm{X}(\xi) - \mathbb{E}_\xi [\bm{X}(\xi)] \|_{\mathrm{F}}^{\mathfrak{p}}]$. The $2$-variance coincides with the variance (i.e., $\mathbb{V}_\xi^{2} [\bm{X}(\xi)] = \mathbb{V}_\xi [\bm{X}(\xi)]$). $\mathbb{E}_\xi [\bm{X}(\xi)|\bm{Y}]$ (resp. $\mathbb{V}_\xi^{\mathfrak{p}} [\bm{X}(\xi)|\bm{Y}]$) denotes the expectation (resp. the $\mathfrak{p}$-variance) of $\bm{X}(\xi)$ conditioned on $\bm{Y}$. When $\bm{\xi}_0, \bm{\xi}_1, \cdots, \bm{\xi}_t$ are independent, we define the total expectation $\mathbb{E}$ by $\mathbb{E} \coloneqq \mathbb{E}_{\bm{\xi}_0} \mathbb{E}_{\bm{\xi}_1} \cdots \mathbb{E}_{\bm{\xi}_t}$. We denote $\xi \sim \mathrm{DU}(N)$ when $\xi$ follows a discrete uniform distribution on $[N]$. The gradient of a differentiable function $f \colon \mathbb{R}^{m \times n} \to \mathbb{R}$ is denoted by $\nabla f \colon \mathbb{R}^{m \times n} \to \mathbb{R}^{m \times n}$.

\section{Nonconvex H\"{o}lder-Smooth ERM}
\label{sec:2}
Let $\bm{W} \in \mathbb{R}^{m \times n}$ be a parameter of a DNN, $S = \{(\bm{x}_1,\bm{y}_1), \ldots, (\bm{x}_N,\bm{y}_N)\}$ be the training set, where data point $\bm{x}_i$ is associated with label $\bm{y}_i$, and $f_i (\cdot) \coloneqq f(\cdot;(\bm{x}_i,\bm{y}_i)) \colon \mathbb{R}^{m \times n} \to \mathbb{R}$ be the loss function corresponding to the $i$-th labeled training data $(\bm{x}_i,\bm{y}_i)$. Empirical risk minimization (ERM) minimizes the empirical risk (ER) defined for all $\bm{W} \in \mathbb{R}^{m \times n}$ as
\begin{align}\label{erm}
f (\bm{W}) 
= \frac{1}{N} \sum_{i=1}^N f_i(\bm{W}).
\end{align} 
This paper considers the following stationary point problem: Find $\bm{W}^\star \in \mathbb{R}^{m \times n}$ such that $\nabla f(\bm{W}^\star) = \bm{O}_{m \times n}$.

\subsection{Assumptions and Examples}
We assume that the loss functions $f_i$ ($i\in [N]$) satisfy the following conditions.

\begin{assumption}\label{assum:1}
{\em
Let $N \in \mathbb{N}$, $\nu \in (0,1]$, $L_i = L_{i}(\nu) > 0$ ($i\in [N]$), and $\mathfrak{p} \in (1,2]$.

{(A1)} $f_i \colon \mathbb{R}^{m \times n} \to \mathbb{R}$ ($i\in [N]$) is $L_i$-H\"{o}lder smooth, i.e., for all $\bm{W}_1, \bm{W}_2 \in \mathbb{R}^{m \times n}$, 
\begin{align*}
\| \nabla f_i (\bm{W}_1) -  \nabla f_i (\bm{W}_2)  \|_{\mathrm{F}}
\leq L_i \| \bm{W}_1 - \bm{W}_2  \|_{\mathrm{F}}^{\nu}
\end{align*}
and $f_i^\star \coloneqq \inf \{ f_i (\bm{W}) \colon \bm{W} \in \mathbb{R}^{m \times n} \} \in \mathbb{R}$.

{(A2)} Let $\xi$ be a random variable that is independent of $\bm{W} \in \mathbb{R}^{m \times n}$. $\nabla f_{\xi} \colon \mathbb{R}^{m \times n} \to \mathbb{R}^{m \times n}$ is the stochastic gradient of $\nabla f$ such that
\begin{enumerate}
\item[(i)] [Unbiasedness of stochastic gradient] for all $\bm{W} \in \mathbb{R}^{m \times n}$, $\mathbb{E}_{\xi}[\nabla f_{\xi}(\bm{W})] = \nabla f(\bm{W})$ and
\item[(ii)] [Boundedness of $\mathfrak{p}$-variance of stochastic gradient] there exists $\sigma \geq 0$ such that, for all $\bm{W} \in \mathbb{R}^{m \times n}$, $\mathbb{V}_{\xi}^{\mathfrak{p}}[\nabla f_{\xi}(\bm{W})] \coloneqq \mathbb{E}_{\xi}[\| \nabla f_{\xi}(\bm{W}) - \mathbb{E}_{\xi}[\nabla f_{\xi}(\bm{W})] \|^{\mathfrak{p}}] \leq \sigma^{\mathfrak{p}}$.
\end{enumerate}
} 
\end{assumption}

The $L_i$-H\"{o}lder smoothness \citep{Hoelder1882} of $f_i$ in Assumption \ref{assum:1}(A1) is used to analyze mini-batch SGD \citep[Assumption 4]{fatkhullin2025can}, \citep[Assumption 2.1]{yamada_2026}, since almost all of the analyses of mini-batch SGD have been based on the following inequality \citep[(2.5)]{Nesterov:2015aa}, \citep[Lemma 1]{Yashtini:2016aa} that is satisfied under $L_i$-H\"{o}lder smoothness of $f_i$: for all $\bm{W}_1, \bm{W}_2 \in \mathbb{R}^{m \times n}$,
\begin{align}\label{g_descent_lemma} 
f_i(\bm{W}_1) 
\leq f_i(\bm{W}_2) + \nabla f_i (\bm{W}_2) \bullet (\bm{W}_1 - \bm{W}_2)
+ \frac{L_i}{1 + \nu}\| \bm{W}_1 - \bm{W}_2 \|_{\mathrm{F}}^{1 + \nu}. 
\end{align}
Inequality \eqref{g_descent_lemma} is called the generalized descent lemma, since this is a generalization of the descent lemma \citep[Lemma 5.7]{beck2017} that is satisfied under $L_i$-smoothness of $f_i$ (Assumption \ref{assum:1}(A1) when $\nu = 1$). If $f_i^\star := \inf \{ f_i(\bm{W}) \colon \bm{W} \in \mathbb{R}^{m \times n} \} = - \infty$ holds, then the loss function $f_i$ corresponding to the $i$-th labeled training data $(\bm{x}_i, \bm{y}_i)$ does not have any global minimizer, which implies that the empirical loss $f$ satisfies $f^\star \coloneqq \inf \{ f(\bm{W}) \colon \bm{W} \in \mathbb{R}^{m \times n} \} = - \infty$. Hence, the interpolation property \citep[Section 4.3.1]{garrigos2024handbookconvergencetheoremsstochastic} (i.e., there exists $\bm{W}^\star \in \mathbb{R}^{m \times n}$ such that, for all $i\in [N]$, $f_i (\bm{W}^\star) = f_i^\star \in \mathbb{R}$) does not hold, whereas the interpolation property does hold for optimization of a linear model with the squared hinge loss for binary classification on linearly separable data \citep[Section 2]{NEURIPS2019_2557911c}. Moreover, in the case where $f$ is convex with $f^\star = - \infty$, there are no stationary points of $f$, which implies that no algorithm ever finds stationary points of $f$. Accordingly, the condition $f_i^\star \coloneqq \inf \{ f_i(\bm{W}) \colon \bm{W} \in \mathbb{R}^{m \times n} \} \in \mathbb{R}$ in (A1) is a natural one for training DNNs including the case where the empirical loss $f$ is the cross-entropy with $\bm{W}^\star \in \mathbb{R}^{m \times n}$ such that $f(\bm{W}^\star) = \inf \{ f(\bm{W}) \colon \bm{W} \in \mathbb{R}^{m \times n} \} \geq 0$.

Stochastic noise is defined by $N_\xi (\bm{W}) \coloneqq \nabla f_\xi (\bm{W}) - \nabla f (\bm{W})$. Assumption \ref{assum:1}(A2) thus ensures that
\begin{align*}
\sigma^{\mathfrak{p}} 
\geq 
\mathbb{V}_{\xi}^{\mathfrak{p}}[\nabla f_{\xi}(\bm{W})] 
\coloneqq 
\mathbb{E}_{\xi}[\| \nabla f_{\xi}(\bm{W}) - \underbrace{\mathbb{E}_{\xi}[\nabla f_{\xi}(\bm{W})}_{\nabla f (\bm{W})}] \|^{\mathfrak{p}}] 
=
\mathbb{E}_{\xi}[ \| N_\xi (\bm{W}) \|^{\mathfrak{p}} ], 
\end{align*}  
which implies that the stochastic noise $N_\xi (\bm{W})$ is heavy-tailed when $\mathfrak{p} \in (1,2)$ \citep[Assumption 1]{DBLP:conf/nips/ZhangKVKRKS20}.
The following example indicates that Assumption \ref{assum:1}(A2) is satisfied when Assumption \ref{assum:1}(A1) holds and the random variable $\xi$ follows the uniform distribution that is used to train DNNs in practice.

\begin{expl}
[Example satisfying Assumption \ref{assum:1}(A2)]\label{exp:1} {\em Suppose that $f_i \colon \mathbb{R}^{m \times n} \to \mathbb{R}$ ($i\in [N]$) satisfies Assumption \ref{assum:1}(A1) with $L_i < 2 L_i^\nu$, $\mathfrak{p} \in (1,2]$, and $\bm{W} \in \mathbb{R}^{m \times n}$ is independent of $\xi \sim \mathrm{DU}(N)$. Then, 
\begin{enumerate}
\item[(i)]
$\displaystyle{
\mathbb{E}_{\xi \sim \mathrm{DU}(N)}[\nabla f_\xi (\bm{W})] = \nabla f (\bm{W})
}$.
\end{enumerate}
Moreover, if $f_i^{\star\star} \coloneqq \sup \{ f_i(\bm{W}) \colon \bm{W} \in \mathbb{R}^{m \times n} \} \in \mathbb{R}$ holds\footnote{It is sufficient that $(\bm{W}_t)$ generated by an optimizer satisfies Assumption \ref{assum:1}. Hence, we may replace the condition $f_i^{\star\star} \coloneqq \sup \{ f_i(\bm{W}) \colon \bm{W} \in \mathbb{R}^{m \times n} \} \in \mathbb{R}$ in Example \ref{exp:1}(ii) with the condition $f_i^{\star\star} \coloneqq \sup \{ f_i(\bm{W}_t) \colon t \in \{0\} \cup \mathbb{N} \} \in \mathbb{R}$. The supremum of $f_i$ tends to $f_i (\bm{W}_0) \in \mathbb{R}$, since $f_i$ satisfies the generalized descent lemma \eqref{g_descent_lemma} and the optimizer has the descent property (see, e.g., Lemma \ref{lem:sgd_descent}).}, then
\begin{enumerate}
\item[(ii)]
$\displaystyle{
\mathbb{V}_{\xi \sim \mathrm{DU}(N)}^{\mathfrak{p}}[\nabla f_\xi (\bm{W})]
\leq \left[ 
\left\{ \frac{1}{N} \sum_{i=1}^N
\left( \frac{2 L_i^{1 + \nu}(f_i^{\star\star} - f_i^{\star})}{2 L_i^\nu - L_i}
+ 
\frac{(1 - \nu) L_i}{(1 + \nu)(2 L_i^\nu - L_i)} 
\right) 
\right\}^{\frac{1}{2}}\right]^{\mathfrak{p}}
\eqqcolon \sigma^{\mathfrak{p}}
}.$
\end{enumerate}
}  
\end{expl}

\begin{proof}
(i) From $\mathrm{P}(\xi = i) = \frac{1}{N}$, we have that $\mathbb{E}_{\xi \sim \mathrm{DU}(N)}[\nabla f_\xi (\bm{W})] \coloneqq \sum_{i=1}^N \nabla f_i (\bm{W}) \mathrm{P}(\xi = i) = \frac{1}{N} \sum_{i=1}^N \nabla f_i (\bm{W}) = \nabla (\frac{1}{N} \sum_{i=1}^N f_i) (\bm{W}) = \nabla f (\bm{W})$.

(ii) Let $i\in [N]$ and $\bm{W}_2 \in \mathbb{R}^{m \times n}$. The generalized descent lemma \eqref{g_descent_lemma} with $\bm{W}_1 \coloneqq \bm{W}_2 - \frac{1}{L_i} \nabla f_i (\bm{W}_2)$ ensures that
\begin{align}\label{descent}
f_i^\star \leq f_i(\bm{W}_1) 
\leq f_i(\bm{W}_2) - \frac{1}{L_i} \| \nabla f_i (\bm{W}_2) \|_{\mathrm{F}}^2
+ \frac{1}{(1 + \nu) L_i^{\nu}} \| \nabla f_i (\bm{W}_2) \|_{\mathrm{F}}^{1 + \nu}.
\end{align}
We apply $a = \| \nabla f_i (\bm{W}_2) \|_{\mathrm{F}}^{1 + \nu}$, $b = 1$, $p = \frac{2}{1 + \nu}$, and $q = \frac{2}{1 - \nu}$ to Young's inequality $ab \leq \frac{a^p}{p} + \frac{b^q}{q}$, where $\frac{1}{p} + \frac{1}{q} = 1$. Then,
\begin{align}\label{young}
\| \nabla f_i (\bm{W}_2) \|_{\mathrm{F}}^{1 + \nu}
\leq 
\frac{1 + \nu}{2} \left( \| \nabla f_i (\bm{W}_2) \|_{\mathrm{F}}^{1 + \nu}  \right)^{\frac{2}{1 + \nu}}
+ \frac{1 - \nu}{2}
= 
\frac{1 + \nu}{2} \| \nabla f_i (\bm{W}_2) \|_{\mathrm{F}}^{2} 
+ \frac{1 - \nu}{2}.
\end{align}
Accordingly, \eqref{descent} and \eqref{young} ensure that 
\begin{align*}
f_i^\star 
&\leq f_i(\bm{W}_2) - \frac{1}{L_i} \| \nabla f_i (\bm{W}_2) \|_{\mathrm{F}}^2
+ \frac{1}{(1 + \nu) L_i^{\nu}} 
\left\{  
\frac{1 + \nu}{2} \| \nabla f_i (\bm{W}_2) \|_{\mathrm{F}}^{2} 
+ \frac{1 - \nu}{2}
\right\}\\
&= f_i(\bm{W}_2) + \frac{L_i - 2 L_i^\nu}{2 L_i^{1 + \nu}} \| \nabla f_i (\bm{W}_2) \|_{\mathrm{F}}^{2} 
+ \frac{1 - \nu}{2 (1 + \nu) L_i^\nu},
\end{align*}
which, together with $f_i^{\star\star} \in \mathbb{R}$ and $L_i - 2 L_i^\nu < 0$, implies that
\begin{align}\label{norm_nabla_f_i}
\| \nabla f_i (\bm{W}_2) \|_{\mathrm{F}}^{2}
\leq
\frac{2 L_i^{1 + \nu}}{2 L_i^\nu - L_i} (f_i^{\star\star} - f_i^{\star})
+ 
\frac{(1 - \nu) L_i}{(1 + \nu)(2 L_i^\nu - L_i)}.
\end{align}
Let $g \colon \mathbb{R} \to \mathbb{R}$ be concave (e.g., $g(x) = x^{\frac{\mathfrak{p}}{2}}$). Jensen's inequality thus ensures that, for all $X \in \mathbb{R}_+$, $\mathbb{E}_\xi [g(X(\xi))] \leq g (\mathbb{E}_\xi [X(\xi)])$. Hence, for all $\bm{X} \in \mathbb{R}^{m \times n}$,
\begin{align*}
\mathbb{V}_{\xi}^{\mathfrak{p}}[\bm{X}(\xi)]
=
\mathbb{E}_\xi \left[ \left( \| \bm{X}(\xi) - \mathbb{E}_\xi [\bm{X}(\xi)] \|_{\mathrm{F}}^{2} \right)^{\frac{\mathfrak{p}}{2}} \right]
\leq 
\left( \mathbb{E}_\xi [ \| \bm{X}(\xi) - \mathbb{E}_\xi [\bm{X}(\xi)] \|_{\mathrm{F}}^{2} ] \right)^{\frac{\mathfrak{p}}{2}}
= \left( \mathbb{V}_{\xi}[\bm{X}(\xi)] \right)^{\frac{\mathfrak{p}}{2}},
\end{align*}
which, together with $\mathbb{V}_{\xi}[\bm{X}(\xi)] = \mathbb{E}_{\xi}[\| \bm{X}(\xi)\|_{\mathrm{F}}^2] - \|\mathbb{E}_{\xi}[ \bm{X}(\xi)]\|_{\mathrm{F}}^2 \leq \mathbb{E}_{\xi}[\| \bm{X}(\xi)\|_{\mathrm{F}}^2]$, implies that
\begin{align}\label{p_variance_exp}
\mathbb{V}_{\xi}^{\mathfrak{p}}[\bm{X}(\xi)]
\leq \left( \mathbb{E}_{\xi}[\| \bm{X}(\xi)\|_{\mathrm{F}}^2] \right)^{\frac{\mathfrak{p}}{2}}.
\end{align} 
Applying $\bm{X}(\xi) = \nabla f_\xi (\bm{W}) = \nabla f_\xi (\bm{W}_2)$ to \eqref{p_variance_exp} and using \eqref{norm_nabla_f_i} lead to the finding that 
\begin{align*}
\mathbb{V}_{\xi \sim \mathrm{DU}(N)}^{\mathfrak{p}}[\nabla f_\xi (\bm{W})]
&\leq \left( \mathbb{E}_{\xi \sim \mathrm{DU}(N)}[\| \nabla f_\xi (\bm{W})\|_{\mathrm{F}}^2] \right)^{\frac{\mathfrak{p}}{2}}
= \left( \sum_{i=1}^N \|\nabla f_i (\bm{W})\|_{\mathrm{F}}^2 \mathrm{P}(\xi = i)   \right)^{\frac{\mathfrak{p}}{2}}\\
&= \left\{ \frac{1}{N} \sum_{i=1}^N
\left( \frac{2 L_i^{1 + \nu}}{2 L_i^\nu - L_i} (f_i^{\star\star} - f_i^{\star})
+ 
\frac{(1 - \nu) L_i}{(1 + \nu)(2 L_i^\nu - L_i)} 
\right) \right\}^{\frac{\mathfrak{p}}{2}},
\end{align*}
which indicates that Assumption \ref{assum:1}(A2)(ii) holds. 
\end{proof}

\subsection{Useful properties of mini-batch gradient}
Let $b \in \mathbb{N}$ be the batch size (the number of samples) and let $\bm{\xi} = (\xi_{1}, \xi_{2}, \cdots, \xi_{b})^\top$ comprise $b$ independent and identically distributed (i.i.d.) variables and be independent of $\bm{W} \in \mathbb{R}^{m \times n}$. Then, the mini-batch gradient of $f$ at $\bm{W}$ is defined by
\begin{align}\label{mini_batch}
\nabla f_{\bm{\xi}}(\bm{W}) \coloneqq 
\frac{1}{b} \sum_{i=1}^b \nabla f_{\xi_{i}}(\bm{W}).
\end{align} 

The following proposition indicates that the mini-batch gradient inherits useful properties of the stochastic gradient such as unbiasedness and boundedness of variance in Assumption \ref{assum:1}(A2). 

\begin{prop}\label{prop:1}
{\em
Suppose that Assumption \ref{assum:1} holds and let $\nabla f_{\bm{\xi}} (\bm{W})$ be defined by \eqref{mini_batch}. Then, the following hold.
\begin{enumerate}
\item[(i)] [Unbiasedness of mini-batch gradient] 
$\displaystyle{
\mathbb{E}_{\bm{\xi}}[\nabla f_{\bm{\xi}}(\bm{W})] = \nabla f(\bm{W})}$;
\item[(ii)] [Boundedness of $\mathfrak{p}$-variance of mini-batch gradient]
$\displaystyle{
\mathbb{V}_{\bm{\xi}}^{\mathfrak{p}}[\nabla f_{\bm{\xi}} (\bm{W})] \leq \frac{2^{2 - \mathfrak{p}}\sigma^{\mathfrak{p}}}{b^{\mathfrak{p}-1}}}$.
\end{enumerate}
} 
\end{prop}

\begin{proof}
(i) From the property of $\mathbb{E}_{\bm{\xi}}$ and Assumption \ref{assum:1}(A2)(i), we have 
\begin{align*}
\mathbb{E}_{\bm{\xi}}[\nabla f_{\bm{\xi}}(\bm{W})]
= 
\mathbb{E}_{\bm{\xi}} \left[
\frac{1}{b} \sum_{i=1}^b \nabla f_{\xi_{i}}(\bm{W})
\right]
= 
\frac{1}{b} \sum_{i=1}^b \mathbb{E}_{\xi_i} [\nabla f_{\xi_{i}}(\bm{W})]
= \frac{1}{b} \sum_{i=1}^b \nabla f (\bm{W})
= \nabla f (\bm{W}).
\end{align*}

(ii) The definition of $\mathbb{V}_{\bm{\xi}}^{\mathfrak{p}}$ and Proposition \ref{prop:1}(i) imply that
\begin{align*}
\mathbb{V}_{\bm{\xi}}^{\mathfrak{p}}[\nabla f_{\bm{\xi}} (\bm{W})]
&= 
\mathbb{E}_{\bm{\xi}} \left[
\left\| \frac{1}{b} \sum_{i=1}^b \left( \nabla f_{\xi_{i}}(\bm{W}) - \nabla f (\bm{W}) \right) \right\|^{\mathfrak{p}}
\right]
= 
\frac{1}{b^{\mathfrak{p}}}
\mathbb{E}_{\bm{\xi}} \left[
\left\|\sum_{i=1}^b \left( \nabla f_{\xi_{i}}(\bm{W}) - \nabla f (\bm{W}) \right) \right\|^{\mathfrak{p}}
\right]\\
&=
\frac{1}{b^{\mathfrak{p}}}
\mathbb{E}_{\bm{\xi}} \Bigg[
\Bigg\|
\underbrace{\sum_{i=2}^b \left( \nabla f_{\xi_{i}}(\bm{W}) - \nabla f (\bm{W}) \right)}_{\bm{W} (\bm{\xi}_{[2:b]})} 
+ 
\underbrace{\left( \nabla f_{\xi_{1}}(\bm{W}) - \nabla f (\bm{W}) \right)}_{\bm{W}_1(\xi_1)} \Bigg\|^{\mathfrak{p}}
\Bigg],
\end{align*}
where $\bm{\xi}_{[2:b]} \coloneqq (\xi_2, \xi_3, \cdots, \xi_b)^\top$. In the case of $\bm{W} (\bm{\xi}_{[2:b]}) = \bm{O}_{m \times n}$ a.s., Assumption \ref{assum:1}(A2) ensures that 
\begin{align*}
\mathbb{V}_{\bm{\xi}}^{\mathfrak{p}}[\nabla f_{\bm{\xi}} (\bm{W})]
=
\frac{1}{b^{\mathfrak{p}}} \mathbb{E}_{\xi_1}[\| \bm{W}_1 \|_{\mathrm{F}}^{\mathfrak{p}}]
= 
\frac{1}{b^{\mathfrak{p}}} \mathbb{V}_{\xi_1}^{\mathfrak{p}} [\nabla f_{\xi_1} (\bm{W}) ]
\leq 
\frac{\sigma^{\mathfrak{p}}}{b^{\mathfrak{p}}} 
\leq 
\frac{2^{2 - \mathfrak{p}} \sigma^{\mathfrak{p}}}{b^{\mathfrak{p}-1}},
\end{align*}
which implies that Proposition \ref{prop:1}(ii) holds. Let us consider the case of $\bm{W} (\bm{\xi}_{[2:b]}) \neq \bm{O}_{m \times n}$ a.s.. From $\| \bm{W} + \bm{W}_1 \|_{\mathrm{F}}^{\mathfrak{p}} \leq \|\bm{W}\|_{\mathrm{F}}^{\mathfrak{p}} + 2^{2 - \mathfrak{p}}\|\bm{W}_1\|_{\mathrm{F}}^{\mathfrak{p}} + \frac{\mathfrak{p}}{\|\bm{W}\|_{\mathrm{F}}^{2 - \mathfrak{p}}} \bm{W} \bullet \bm{W}_1$ ($\bm{W}_1, \bm{W} (\neq \bm{O}_{m \times n}) \in \mathbb{R}^{m \times n}$) and the independence of $\bm{W}_1 (\xi_1)$ and $\bm{W} (\bm{\xi}_{[2:b]})$, we have 
\begin{align}\label{expan_1}
\begin{split}
\mathbb{E}_{\bm{\xi}} [\|\bm{W} + \bm{W}_1 \|_{\mathrm{F}}^{\mathfrak{p}}]
&\leq
\mathbb{E}_{\bm{\xi}_{[2:b]}} [ \|\bm{W}\|_{\mathrm{F}}^{\mathfrak{p}}]
+
2^{2 - \mathfrak{p}} \mathbb{E}_{\xi_1} [\|\bm{W}_1\|_{\mathrm{F}}^{\mathfrak{p}}]
+
\mathbb{E}_{\bm{\xi}} \left[ \frac{\mathfrak{p} \bm{W}}{\|\bm{W}\|_{\mathrm{F}}^{2 - \mathfrak{p}}}
\bullet \bm{W}_1 \right]\\
&=
\mathbb{E}_{\bm{\xi}_{[2:b]}} [ \|\bm{W}\|_{\mathrm{F}}^{\mathfrak{p}}]
+
2^{2 - \mathfrak{p}} \mathbb{E}_{\xi_1} [\|\bm{W}_1\|_{\mathrm{F}}^{\mathfrak{p}}]
+
\mathbb{E}_{\bm{\xi}_{[2:b]}} \left[ \frac{\mathfrak{p} \bm{W}}{\|\bm{W}\|_{\mathrm{F}}^{2 - \mathfrak{p}}}\right]
\bullet \mathbb{E}_{\xi_1} [\bm{W}_1],
\end{split} 
\end{align}
which, together with $\mathbb{E}_{\xi_1}[\bm{W}_1] = \mathbb{E}_{\xi_1}[\nabla f_{\xi_{1}}(\bm{W})] - \nabla f (\bm{W}) = \nabla f (\bm{W}) - \nabla f (\bm{W}) = \bm{O}_{m \times n}$ (by Assumption \ref{assum:1}(A2)(i)), implies that
\begin{align}\label{repeat_1}
&\mathbb{E}_{\bm{\xi}} [\|\bm{W} + \bm{W}_1 \|_{\mathrm{F}}^{\mathfrak{p}}]
\leq
\mathbb{E}_{\bm{\xi}_{[2:b]}} [ \|\bm{W}\|_{\mathrm{F}}^{\mathfrak{p}}]
+
2^{2 - \mathfrak{p}} \mathbb{E}_{\xi_1} [\|\bm{W}_1\|_{\mathrm{F}}^{\mathfrak{p}}]\\
&=
\mathbb{E}_{\bm{\xi}_{[2:b]}} \Bigg[ \Bigg\|
\underbrace{\sum_{i=3}^b \left( \nabla f_{\xi_{i}}(\bm{W}) - \nabla f (\bm{W}) \right)}_{\bm{W}(\bm{\xi}_{[3:b]})}
+
\underbrace{\left( \nabla f_{\xi_{2}}(\bm{W}) - \nabla f (\bm{W}) \right)}_{\bm{W}_2 (\xi_2)}
\Bigg\|_{\mathrm{F}}^{\mathfrak{p}} \Bigg]
+
2^{2 - \mathfrak{p}} \mathbb{E}_{\xi_1} [\|\bm{W}_1\|_{\mathrm{F}}^{\mathfrak{p}}]. \nonumber 
\end{align}
If $\bm{W}(\bm{\xi}_{[3:b]}) = \bm{O}_{m \times n}$ a.s., then Assumption \ref{assum:1}(A2) and the condition $b \geq 3$ imply that 
\begin{align*}
\mathbb{V}_{\bm{\xi}}^{\mathfrak{p}}[\nabla f_{\bm{\xi}} (\bm{W})]
\leq 
\frac{1}{b^{\mathfrak{p}}} 
\left( 
\mathbb{E}_{\xi_2}[\| \bm{W}_2 \|_{\mathrm{F}}^{\mathfrak{p}}]
+
2^{2 - \mathfrak{p}} \mathbb{E}_{\xi_1} [\|\bm{W}_1\|_{\mathrm{F}}^{\mathfrak{p}}]
\right)
\leq 
\frac{2 \cdot 2^{2 - \mathfrak{p}} \sigma^{\mathfrak{p}}}{b^{\mathfrak{p}}}
\leq 
\frac{b \cdot 2^{2 - \mathfrak{p}} \sigma^{\mathfrak{p}}}{b^{\mathfrak{p}}}
= 
\frac{2^{2 - \mathfrak{p}} \sigma^{\mathfrak{p}}}{b^{\mathfrak{p}-1}}.
\end{align*}
Hence, we may assume $\bm{W}(\bm{\xi}_{[i : b]}) \neq \bm{O}_{m \times n}$ a.s. ($i \in [3:b]$). A similar argument to the one above for \eqref{repeat_1} leads to 
\begin{align*}
\mathbb{E}_{\bm{\xi}} [\|\bm{W} + \bm{W}_1 \|_{\mathrm{F}}^{\mathfrak{p}}]
&\leq
\mathbb{E}_{\xi_b} [\| \nabla f_{\xi_{b}}(\bm{W}) - \nabla f (\bm{W}) \|_{\mathrm{F}}^{\mathfrak{p}} ]
+ 2^{2 - \mathfrak{p}} \sum_{i=1}^{b-1} \mathbb{E}_{\xi_{i}} [\|\nabla f_{\xi_{i}}(\bm{W}) - \nabla f (\bm{W})\|_{\mathrm{F}}^{\mathfrak{p}}]\\
&\eqqcolon
\mathbb{V}_{\xi_b}^{\mathfrak{p}} [ \nabla f_{\xi_{b}}(\bm{W}) ]
+ 
2^{2 - \mathfrak{p}} \sum_{i=1}^{b-1} \mathbb{V}_{\xi_i}^{\mathfrak{p}} [ \nabla f_{\xi_{i}}(\bm{W}) ].
\end{align*}
Accordingly, from Assumption \ref{assum:1}(A2)(ii),
\begin{align*}
\mathbb{V}_{\bm{\xi}}^{\mathfrak{p}}[\nabla f_{\bm{\xi}} (\bm{W})]
\leq 
\frac{1}{b^{\mathfrak{p}}}
\left( \sigma^{\mathfrak{p}} + 2^{2 - \mathfrak{p}} (b-1) \sigma^{\mathfrak{p}} \right)
\leq 
\frac{1}{b^{\mathfrak{p}}}
\left( 2^{2 - \mathfrak{p}} \sigma^{\mathfrak{p}} + 2^{2 - \mathfrak{p}} (b-1) \sigma^{\mathfrak{p}} \right)
= 
\frac{2^{2 - \mathfrak{p}} \sigma^{\mathfrak{p}}}{b^{\mathfrak{p}-1}}.
\end{align*}
This completes the proof.
\end{proof}

\section{Mini-batch SGD}
\label{sec:3}
First, we will consider the following mini-batch SGD to minimize $f$ defined by \eqref{erm} under Assumption \ref{assum:1}: Given an initial point $\bm{W}_0 \in \mathbb{R}^{m \times n}$,
\begin{align}\label{SGD}
\begin{split}
&\text{[Mini-batch SGD]}\\
&\bm{W}_{t+1} 
= \bm{W}_t + \eta_t \bm{D}_t^{\mathrm{SGD}}
= \bm{W}_t - \eta_t \nabla f_{\bm{\xi}_t} (\bm{W}_t) 
= \bm{W}_t - \frac{\eta_t}{b_t} \sum_{i=1}^{b_t} \nabla f_{\xi_{t,i}} (\bm{W}_t),
\end{split}
\end{align}
where $\eta_t > 0$ is the step size, $b_t \in \mathbb{N}$ is the batch size, $\bm{\xi}_t = (\xi_{t,1}, \cdots, \xi_{t, b_t})^\top$ comprises $b_t$ i.i.d. variables and is independent of $\bm{W}_t$, and $\bm{D}_t^{\mathrm{SGD}} \coloneqq - \nabla f_{\bm{\xi}_t} (\bm{W}_t)$ is the search direction of mini-batch SGD. We may in theory assume sampling with replacement. In sampling with replacement, even if the batch size $b_t$ exceeds $N$, $\nabla f_{\bm{\xi}_t} \neq \nabla f$ holds in general. Hence, to examine the convergence of mini-batch optimizers under sampling with replacement, we can use $b_t \to + \infty$ ($t \to + \infty$).

Although the previously reported results in \citep[Theorem 4]{fatkhullin2025can} and \citep[Theorem 3.5]{yamada_2026} indicated convergence of mini-batch SGD under H\"{o}lder smoothness, this section presents it in comparison with the convergence of the Muon optimizer. 

\subsection{Descent property}
The following lemma gives the descent property of mini-batch SGD \eqref{SGD} to minimize $f$ defined by \eqref{erm}. 

\begin{lem}\label{lem:sgd_descent}
{\em 
Let $(\bm{W}_t)$ be a sequence generated by mini-batch SGD \eqref{SGD} under Assumption \ref{assum:1} and let $\bm{\xi}_{[t-1]} \coloneqq \{ \bm{\xi}_0, \cdots, \bm{\xi}_{t-1} \}$. Under the condition $\nabla f (\bm{W}_t) \neq \bm{O}_{m \times n}$ for all $t \in \{0\} \cup \mathbb{N}$,
\begin{enumerate}
\item[(i)]
$\displaystyle{
\mathbb{E}_{\bm{\xi}_t} \left[\nabla f (\bm{W}_t) \bullet \bm{D}_t^{\mathrm{SGD}}|\bm{\xi}_{[t-1]} \right] = - \|\nabla f (\bm{W}_t)\|_{\mathrm{F}}^2 < 0
}$.
\end{enumerate}
Let $L \coloneqq \frac{1}{N} \sum_{i=1}^N L_i$. If $1 + \nu \leq \mathfrak{p}$ and $\eta_t^{\nu} < \frac{2}{L}$ hold, then
\begin{enumerate}
\item[(ii)] 
$\displaystyle{
\mathbb{E}_{\bm{\xi}_t} \left[f (\bm{W}_{t+1}) |\bm{\xi}_{[t-1]} \right] 
<
f(\bm{W}_{t}) 
+ \frac{(1 - \nu) L \eta_t^{1 + \nu}}{2 (1 + \nu)}
+ \frac{2^{3 - (\nu + \mathfrak{p})} L \sigma^{\mathfrak{p}} \eta_t^{1 + \nu}}{(1 + \nu) b_t^{\mathfrak{p}-1}}. 
}$
\end{enumerate}
This implies that, if $1 + \nu \leq \mathfrak{p}$ holds and if $(\eta_t^{1 + \nu})$ and $(\frac{\eta_t^{1 + \nu}}{b_t^{\mathfrak{p}-1}})$ converge to $0$, then, for all $\epsilon > 0$, there exists $t_0 \in \mathbb{N}$ such that, for all $t \geq t_0$, $\eta_t^{\nu} < \frac{2}{L}$ and $\mathbb{E}_{\bm{\xi}_t} \left[f (\bm{W}_{t+1}) |\bm{\xi}_{[t-1]} \right] < f(\bm{W}_t) + \epsilon$.}
\end{lem}

Lemma \ref{lem:sgd_descent}(i) indicates that the search direction $\bm{D}_t^{\mathrm{SGD}} = - \nabla f_{\bm{\xi}_t} (\bm{W}_t)$ is a descent direction of $f$ in the sense of the conditional expectation $\mathbb{E}_{\bm{\xi}_t}[\cdot | \bm{\xi}_{[t-1]}]$. However, alone, the property of the descent direction $\bm{D}_t^{\mathrm{SGD}}$ does not guarantee minimization of $f$, since using a large step size $\eta_t$ would increase $f$. Hence, in order to minimize $f$ by using mini-batch SGD \eqref{SGD}, we will set a small step size $\eta_t$. In fact, Lemma \ref{lem:sgd_descent}(ii) indicates that, if we set a diminishing step size $\eta_t$ (e.g., $\eta_t$ decreases with each epoch), then mini-batch SGD \eqref{SGD} decreases $f$ in the sense that $\mathbb{E}_{\bm{\xi}_t} \left[f (\bm{W}_{t+1}) |\bm{\xi}_{[t-1]} \right] < f(\bm{W}_t) + \epsilon \approx f(\bm{W}_t)$.

\begin{proof}
\textbf{of Lemma \ref{lem:sgd_descent}} (i) The property of $\mathbb{E}_{\bm{\xi}_t}$ and Proposition \ref{prop:1}(i) imply that, for all $t \in \{0\} \cup \mathbb{N}$, 
\begin{align}\label{eq:1}
\begin{split}
\mathbb{E}_{\bm{\xi}_t} \left[\nabla f (\bm{W}_t) \bullet \bm{D}_t^{\mathrm{SGD}}|\bm{\xi}_{[t-1]} \right]
&= - \mathbb{E}_{\bm{\xi}_t} \left[\nabla f (\bm{W}_t) \bullet \nabla f_{\bm{\xi}_t} (\bm{W}_t) |\bm{\xi}_{[t-1]} \right]\\ 
&= - \nabla f (\bm{W}_t) \bullet 
\underbrace{\mathbb{E}_{\bm{\xi}_t} \left[ \nabla f_{\bm{\xi}_t} (\bm{W}_t) |\bm{\xi}_{[t-1]} \right]}_{\nabla f (\bm{W}_t)} 
= - \|\nabla f (\bm{W}_t)\|_{\mathrm{F}}^2.
\end{split}
\end{align}

(ii) Summing the generalized descent lemma \eqref{g_descent_lemma} for a H\"{o}lder smooth function $f_i$ ($i\in [N]$) ensures that, for all $\bm{W}_1, \bm{W}_2 \in \mathbb{R}^{m \times n}$, 
\begin{align*} 
\sum_{i=1}^N f_i(\bm{W}_1) 
\leq 
\sum_{i=1}^N f_i(\bm{W}_2) + \sum_{i=1}^N \nabla f_i (\bm{W}_2) \bullet (\bm{W}_1 - \bm{W}_2)
+ \frac{\sum_{i=1}^N L_i}{1 + \nu}\| \bm{W}_1 - \bm{W}_2 \|_{\mathrm{F}}^{1 + \nu}, 
\end{align*}
which, together with the definition \eqref{erm} of $f$, implies that, for all $\bm{W}_1, \bm{W}_2 \in \mathbb{R}^{m \times n}$, 
\begin{align}\label{g_descent_lemma_sum}
f(\bm{W}_1) 
\leq 
f(\bm{W}_2) + \nabla f (\bm{W}_2) \bullet (\bm{W}_1 - \bm{W}_2)
+ \frac{L}{1 + \nu}\| \bm{W}_1 - \bm{W}_2 \|_{\mathrm{F}}^{1 + \nu}, 
\end{align}
where $L \coloneqq \frac{1}{N} \sum_{i=1}^N L_i$. Applying $\bm{W}_1 = \bm{W}_{t+1}$ and $\bm{W}_2 = \bm{W}_t$ to \eqref{g_descent_lemma_sum} and using $\bm{W}_{t+1} - \bm{W}_t = \eta_t \bm{D}_t^{\mathrm{SGD}}$ imply that, for all $t \in \mathbb{N}$,
\begin{align*}
&\mathbb{E}_{\bm{\xi}_t} \left[ f(\bm{W}_{t+1})|\bm{\xi}_{[t-1]} \right]\\
&\leq 
f(\bm{W}_t) 
+ \eta_t 
\underbrace{\mathbb{E}_{\bm{\xi}_t} \left[ \nabla f (\bm{W}_t) \bullet \bm{D}_t^{\mathrm{SGD}} |\bm{\xi}_{[t-1]} \right]}_{= - \|\nabla f (\bm{W}_t)\|_{\mathrm{F}}^2 \text{ } \because \text{ } \eqref{eq:1}}
+ \frac{L \eta_t^{1 + \nu}}{1 + \nu} 
\underbrace{\mathbb{E}_{\bm{\xi}_t} \left[
\left\| \bm{D}_t^{\mathrm{SGD}} \right\|_{\mathrm{F}}^{1 + \nu} \Big| \bm{\xi}_{[t-1]} \right]}_{D_t}.
\end{align*}
Using the same proof technique \eqref{expan_1} (i.e., the expansion of the $\mathfrak{p}$-th power) as in Proposition \ref{prop:1}(ii) and the same proof techniques \eqref{young} and \eqref{p_variance_exp} (i.e., Jensen's inequality and Young's inequality) as in Example \ref{exp:1}, we can evaluate an upper bound of $D_t$. From $\nabla f (\bm{W}_t) \neq \bm{O}_{m \times n}$ for all $t$ and the expansion of the $(1 + \nu)$-th power, we have 
\begin{align}\label{comp_d_t}
\begin{split}
D_t
&\leq
\|\nabla f (\bm{W}_t) \|_{\mathrm{F}}^{1 + \nu}
+
2^{1 - \nu} \mathbb{E}_{\bm{\xi}_t} \left[\|\nabla f_{\bm{\xi}_t} (\bm{W}_t) - \nabla f (\bm{W}_t) \|_{\mathrm{F}}^{1 + \nu} |\bm{\xi}_{[t-1]} \right]\\
&\quad +
\mathbb{E}_{\bm{\xi}_t} \left[ \frac{(1 + \nu) \nabla f (\bm{W}_t)}{\|\nabla f (\bm{W}_t)\|_{\mathrm{F}}^{1 - \nu}}
\bullet (\nabla f_{\bm{\xi}_t} (\bm{W}_t) - \nabla f (\bm{W}_t)) \Bigg|\bm{\xi}_{[t-1]} \right]\\
&= 
\|\nabla f (\bm{W}_t) \|_{\mathrm{F}}^{1 + \nu}
+
2^{1 - \nu} \mathbb{V}_{\bm{\xi}_t}^{1 + \nu} \left[\nabla f_{\bm{\xi}_t} (\bm{W}_t)|\bm{\xi}_{[t-1]} \right]\\
&\quad +
\frac{(1 + \nu) \nabla f (\bm{W}_t)}{\|\nabla f (\bm{W}_t)\|_{\mathrm{F}}^{1 - \nu}}
\bullet 
\mathbb{E}_{\bm{\xi}_t} \left[ \nabla f_{\bm{\xi}_t} (\bm{W}_t) - \nabla f (\bm{W}_t) |\bm{\xi}_{[t-1]} \right]\\
&\leq \|\nabla f (\bm{W}_t) \|_{\mathrm{F}}^{1 + \nu}
+
2^{1 - \nu} 
\left( \mathbb{V}_{\bm{\xi}_t}^{\mathfrak{p}} \left[\nabla f_{\bm{\xi}_t} (\bm{W}_t)|\bm{\xi}_{[t-1]} \right] \right)^{\frac{1 + \nu}{\mathfrak{p}}},
\end{split}
\end{align}
where the relation $\mathbb{E}_{\bm{\xi}_t} [ \nabla f_{\bm{\xi}_t} (\bm{W}_t) - \nabla f (\bm{W}_t) |\bm{\xi}_{[t-1]} ] = \mathbb{E}_{\bm{\xi}_t} [ \nabla f_{\bm{\xi}_t} (\bm{W}_t)|\bm{\xi}_{[t-1]}] - \nabla f (\bm{W}_t) = \bm{O}_{m \times n}$ comes from Proposition \ref{prop:1}(i) and $\mathbb{V}_{\bm{\xi}_t}^{1 + \nu} [\nabla f_{\bm{\xi}_t} (\bm{W}_t)|\bm{\xi}_{[t-1]}] \leq ( \mathbb{V}_{\bm{\xi}_t}^{\mathfrak{p}} [\nabla f_{\bm{\xi}_t} (\bm{W}_t)|\bm{\xi}_{[t-1]} ])^{\frac{1 + \nu}{\mathfrak{p}}}$ comes from Jensen's inequality with a concave function $g (x) = x^{\frac{1 + \nu}{\mathfrak{p}}}$ by $1 + \nu \leq \mathfrak{p}$. Moreover, Young's inequality and Proposition \ref{prop:1}(ii) ensure that
\begin{align}\label{comp_d_t_1}
D_t \leq 
\frac{1 + \nu}{2} \| \nabla f (\bm{W}_t)\|_{\mathrm{F}}^2 + \frac{1 - \nu}{2}
+ \frac{2^{3 - (\nu + \mathfrak{p})}\sigma^{\mathfrak{p}}}{b_t^{\mathfrak{p}-1}}.
\end{align}
Accordingly, for all $t \in \{0\} \cup \mathbb{N}$, 
\begin{align}\label{ineq_sgd}
\begin{split}
&\mathbb{E}_{\bm{\xi}_t} \left[ f(\bm{W}_{t+1})|\bm{\xi}_{[t-1]} \right]\\
&\leq
f(\bm{W}_t)
- \eta_t \| \nabla f (\bm{W}_t)\|_{\mathrm{F}}^2
+ \frac{L \eta_t^{1 + \nu}}{1 + \nu}
\left\{ \frac{1 + \nu}{2} \| \nabla f (\bm{W}_t)\|_{\mathrm{F}}^2 + \frac{1 - \nu}{2}
+ \frac{2^{3 - (\nu + \mathfrak{p})}\sigma^{\mathfrak{p}}}{b_t^{\mathfrak{p}-1}} \right\}\\
&=
f(\bm{W}_t) - \eta_t \left( 1 - \frac{L \eta_t^\nu}{2}   \right) \| \nabla f (\bm{W}_t)\|_{\mathrm{F}}^2
+ \frac{(1 - \nu) L \eta_t^{1 + \nu}}{2 (1 + \nu)}
+ \frac{2^{3 - (\nu + \mathfrak{p})} L \sigma^{\mathfrak{p}} \eta_t^{1 + \nu}}{(1 + \nu) b_t^{\mathfrak{p}-1}},
\end{split}
\end{align}
which, together with $\eta_t^\nu < \frac{2}{L}$, completes the proof.
\end{proof}

\subsection{Convergence}
The following is a convergence analysis of mini-batch SGD \eqref{SGD} under Assumption \ref{assum:1}. Theorem \ref{thm:sgd_convergence}, together with Lemma \ref{lem:sgd_descent}, indicates that mini-batch SGD converges to a stationary point of $f$ (a local minimizer of $f$ or a saddle point of $f$) under the conditions in \eqref{sgd_condition}, which are stronger than the convergence of $(\eta_t^{1 + \nu})$ and $(\frac{\eta_t^{1 + \nu}}{b_t^{\mathfrak{p}-1}})$ to $0$ in Lemma \ref{lem:sgd_descent}(ii).

\begin{thm}\label{thm:sgd_convergence}
Let $(\bm{W}_t)$ be a sequence generated by mini-batch SGD \eqref{SGD} under Assumption \ref{assum:1}. If $1 + \nu \leq \mathfrak{p}$ holds and if $(\eta_t)$ and $(b_t)$ satisfy 
\begin{align}\label{sgd_condition}
\sum_{t=0}^{+\infty} \eta_t = + \infty, \text{ }
\sum_{t=0}^{+\infty} \eta_t^{1 + \nu} < + \infty, \text{ } 
\sum_{t=0}^{+\infty} \frac{\eta_t^{1 + \nu}}{b_t^{\mathfrak{p} - 1}} < + \infty,
\end{align}
then $(\nabla f (\bm{W}_t))$ converges to $\bm{O}_{m \times n}$ almost surely in the sense of the limit inferior. 
\end{thm}

\begin{proof}
\textbf{of Theorem \ref{thm:sgd_convergence}} Inequality \eqref{ineq_sgd}, \eqref{sgd_condition}, and the super martingale convergence theorem \citep[Proposition 8.2.10]{bert}  give
\begin{align*}
\sum_{t=0}^{+ \infty} \eta_t \left( 1 - \frac{L \eta_t^\nu}{2}   \right) \| \nabla f (\bm{W}_t)\|_{\mathrm{F}}^2 < + \infty \text{ a.s.}.
\end{align*}
From $\eta_t \to 0$ ($t \to + \infty$), there exist $t_1 \in \mathbb{N}$ and $\overline{\eta} > 0$ such that, for all $t \geq t_1$, $\eta_t^\nu \leq \overline{\eta}^\nu < \frac{2}{L}$. Hence, we have 
\begin{align*}
\sum_{t= t_1}^{+ \infty} \eta_t \| \nabla f (\bm{W}_t)\|_{\mathrm{F}}^2 < + \infty \text{ a.s.},
\end{align*}
which, together with $\sum_{t=0}^{+ \infty} \eta_t = + \infty$, implies that $\liminf_{t \to + \infty} \| \nabla f (\bm{W}_t)\|_{\mathrm{F}}^2 = 0$, i.e., 
\begin{align*}
\liminf_{t \to + \infty} \| \nabla f (\bm{W}_t)\|_{\mathrm{F}} = 0 \text{ a.s.}.
\end{align*}
This completes the proof.
\end{proof}

\subsection{Convergence Rate}
\subsubsection{Upper convergence bound}
We show an upper convergence rate of mini-batch SGD \eqref{SGD} that converges in the Ces\`{a}ro mean.

\begin{thm}\label{thm:sgd_convergence_rate_upper}
Let $(\bm{W}_t)$ be the sequence generated by mini-batch SGD \eqref{SGD} with $1 + \nu \leq \mathfrak{p}$ and $(\eta_t)$ and $(b_t)$ satisfying \eqref{sgd_condition} under Assumption \ref{assum:1}. Then, for all $T \in \mathbb{N}$, the mean of $(\bm{W}_t)_{t = t_1}^{T + t_1 -1}$ satisfies 
\begin{align*}
&\frac{1}{\sum_{t= t_1}^{T + t_1 -1} \eta_t}
\sum_{t= t_1}^{T + t_1 -1} \eta_t 
\mathbb{E} [\| \nabla f (\bm{W}_t)\|_{\mathrm{F}}^2]
= O \left( \frac{1}{\sum_{t= t_1}^{T + t_1 -1} \eta_t} \right) \\
&\leq
\frac{C_1 (\nu)}{\sum_{t= t_1}^{T + t_1 -1} \eta_t}
+ \frac{C_2 (\nu)}{\sum_{t= t_1}^{T + t_1 -1} \eta_t} \sum_{t= t_1}^{+ \infty} \eta_t^{1 + \nu}
+ \frac{C_3 (\nu, \mathfrak{p}, \sigma)}{\sum_{t= t_1}^{T + t_1 -1} \eta_t} \sum_{t= t_1}^{+ \infty} \frac{\eta_t^{1 + \nu}}{b_t^{\mathfrak{p}-1}},
\end{align*}
where $L \coloneqq \frac{1}{N} \sum_{i=1}^N L_i$, $t_1 \in \mathbb{N}$ and $\overline{\eta} > 0$ are such that, for all $t \geq t_1$, $\eta_t^\nu \leq \overline{\eta}^\nu < \frac{2}{L}$, $f^\star \in \mathbb{R}$ is such that, for all $\bm{W} \in \mathbb{R}^{m \times n}$, $f(\bm{W}) \geq f^\star$, and 
\begin{align*}
C_1 (\nu) \coloneqq \frac{2(\mathbb{E} [f(\bm{W}_{t_1})] - f^{\star})}{2 - L \overline{\eta}^\nu}, \text{ }
C_2 (\nu) \coloneqq \frac{(1 - \nu) L}{(1 + \nu)(2 - L \overline{\eta}^\nu)}, \text{ } 
C_3 (\nu, \mathfrak{p}, \sigma) \coloneqq \frac{2^{4 - (\nu + \mathfrak{p})} L \sigma^{\mathfrak{p}}}{(1 + \nu)(2 - L \overline{\eta}^\nu)}.
\end{align*}
\end{thm}

\begin{proof}
\textbf{of Theorem \ref{thm:sgd_convergence_rate_upper}} From $\eta_t \to 0$ ($t \to + \infty$), there exist $t_1 \in \mathbb{N}$ and $\overline{\eta} > 0$ such that, for all $t \geq t_1$, $\eta_t^\nu \leq \overline{\eta}^\nu < \frac{2}{L}$. Since \eqref{ineq_sgd} holds for all $t \geq t_1$, we can take the total expectation $\mathbb{E} = \mathbb{E}_t \coloneqq \mathbb{E}_{\bm{\xi}_{t_1}} \cdots \mathbb{E}_{\bm{\xi}_{t}}$ to \eqref{ineq_sgd}. Hence, for all $t \geq t_1$,
\begin{align*}
\frac{2 - L \overline{\eta}^\nu}{2} \eta_t 
\mathbb{E} [\| \nabla f (\bm{W}_t)\|_{\mathrm{F}}^2]
&\leq 
\mathbb{E} [f(\bm{W}_t)] 
-
\mathbb{E} [ f(\bm{W}_{t+1})]
+ \frac{(1 - \nu) L \eta_t^{1 + \nu}}{2 (1 + \nu)}
+ \frac{2^{3 - (\nu + \mathfrak{p})} L \sigma^{\mathfrak{p}} \eta_t^{1 + \nu}}{(1 + \nu) b_t^{\mathfrak{p}-1}}.
\end{align*}
Let $T \in \mathbb{N}$. Summing the above inequality from $t = t_1$ to $t = T + t_1 -1$ and invoking Assumption \ref{assum:1}(A1) (the existence of $f_i^\star$ ($i\in [N]$)) together ensure that
\begin{align*}
&\frac{2 - L \overline{\eta}^\nu}{2} \sum_{t= t_1}^{T + t_1 -1} \eta_t 
\mathbb{E} [\| \nabla f (\bm{W}_t)\|_{\mathrm{F}}^2]\\
&\leq 
\mathbb{E} [f(\bm{W}_{t_1})] 
-
\mathbb{E} [ f(\bm{W}_{T + t_1})]
+ \frac{(1 - \nu) L}{2 (1 + \nu)} \sum_{t= t_1}^{T + t_1 -1} \eta_t^{1 + \nu}
+ \frac{2^{3 - (\nu + \mathfrak{p})} L \sigma^{\mathfrak{p}}}{(1 + \nu)} \sum_{t= t_1}^{T + t_1 -1} \frac{\eta_t^{1 + \nu}}{b_t^{\mathfrak{p}-1}}\\
&\leq 
\mathbb{E} [f(\bm{W}_{t_1})] 
-
f^{\star}
+ \frac{(1 - \nu) L}{2 (1 + \nu)} \sum_{t= t_1}^{T + t_1 -1} \eta_t^{1 + \nu}
+ \frac{2^{3 - (\nu + \mathfrak{p})} L \sigma^{\mathfrak{p}}}{(1 + \nu)} \sum_{t= t_1}^{T + t_1 -1} \frac{\eta_t^{1 + \nu}}{b_t^{\mathfrak{p}-1}},
\end{align*}
where $f^\star$ satisfies $f(\bm{W}) = \frac{1}{N} \sum_{i=1}^N f_i (\bm{W}) \geq \frac{1}{N} \sum_{i=1}^N f_i^\star \eqqcolon f^\star$ for all $\bm{W} \in \mathbb{R}^{m \times n}$. Accordingly, for all $T \in \mathbb{N}$,
\begin{align*}
&\frac{1}{\sum_{t= t_1}^{T + t_1 -1} \eta_t}
\sum_{t= t_1}^{T + t_1 -1} \eta_t 
\mathbb{E} [\| \nabla f (\bm{W}_t)\|_{\mathrm{F}}^2]\\
&\leq
\frac{2(\mathbb{E} [f(\bm{W}_{t_1})] - f^{\star})}{2 - L \overline{\eta}^\nu}
\frac{1}{\sum_{t= t_1}^{T + t_1 -1} \eta_t}
+ \frac{(1 - \nu) L}{(1 + \nu)(2 - L \overline{\eta}^\nu)} \frac{1}{\sum_{t= t_1}^{T + t_1 -1} \eta_t} \sum_{t= t_1}^{T + t_1 -1} \eta_t^{1 + \nu}\\
&\quad + \frac{2^{4 - (\nu + \mathfrak{p})} L \sigma^{\mathfrak{p}}}{(1 + \nu)(2 - L \overline{\eta}^\nu)} \frac{1}{\sum_{t= t_1}^{T + t_1 -1} \eta_t} \sum_{t= t_1}^{T + t_1 -1} \frac{\eta_t^{1 + \nu}}{b_t^{\mathfrak{p}-1}},
\end{align*}
which completes the proof.
\end{proof}

\subsubsection{Lower convergence bound}
Lemma \ref{lem:sgd_descent} indicates that, for sufficiently large steps $t$, mini-batch SGD \eqref{SGD} with an appropriate step size and batch size decreases $f$ in the sense that $\mathbb{E}_{\bm{\xi}_t} \left[f (\bm{W}_{t+1}) |\bm{\xi}_{[t-1]} \right] < f(\bm{W}_t) + \epsilon \approx f(\bm{W}_t)$. Moreover, Theorem \ref{thm:sgd_convergence} ensures convergence of mini-batch SGD \eqref{SGD} to a stationary point of $f$. When the empirical loss $f$ defined by \eqref{erm} is a nonconvex function with many local minimizers, we may assume from the above results in Lemma \ref{lem:sgd_descent} and Theorem \ref{thm:sgd_convergence} that mini-batch SGD \eqref{SGD} converges to a local minimizer, denoted by $\bm{W}^\star$. Hence, we assume the following: 

\begin{assumption}\label{assum:2}
{\em 
\text{ }

(A3) $f$ is convex in a neighborhood of a convergent point $\bm{W}^\star$; 

(A4) There exists $t_2 \in \mathbb{N}$ such that $C_4 \coloneqq \inf \{\mathbb{E}[f(\bm{W}_{t_2})] - \mathbb{E}[f(\bm{W}_t)] \colon t \geq t_2 \} \geq 0$.
}
\end{assumption} 

When Assumption \ref{assum:2}(A3) holds, we have that, for all $\bm{W}$ in a neighborhood $N(\bm{W}^\star ; B) \coloneqq \{ \bm{W} \colon \|\bm{W} - \bm{W}^\star \|_{\mathrm{F}} \leq B \}$ of a stationary point $\bm{W}^\star$, where $B > 0$, $f(\bm{W}) \geq f(\bm{W}^\star) + \nabla f (\bm{W}^\star) \bullet (\bm{W} - \bm{W}^\star) = f(\bm{W}^\star)$, which implies that $\bm{W}^\star$ is a local minimizer of $f$. Hence, Assumption \ref{assum:2}(A3) is a slightly stronger condition than the one in which the convergent point is a local minimizer of $f$. Theorem \ref{thm:sgd_convergence} ensures that, for a sufficiently large $s$, $(\bm{W}_t)_{t=s}^{+ \infty} \subset N(\bm{W}^\star ; B)$. Hence, $f$ is convex at $\bm{W}_t$ $(t \geq s)$ under Assumption \ref{assum:2}(A3). Let us consider Assumption \ref{assum:2}(A4). Under Assumption \ref{assum:2}(A3), Theorem \ref{thm:sgd_convergence} implies that, for all $\epsilon > 0$, there exists $t_2 \in \mathbb{N}$ such that, for all $t \geq t_2$, $\| \nabla f (\bm{W}_t)\|_{\mathrm{F}} \leq \epsilon$ and $f$ is convex on $N(\bm{W}^\star; B)$ ($\ni \bm{W}_t$). The Cauchy-Schwarz inequality thus ensures that $f(\bm{W}_{t_2}) \geq f(\bm{W}_t) + \nabla f (\bm{W}_t) \bullet (\bm{W}_{t_2} - \bm{W}_t) \geq f(\bm{W}_t) - \|\nabla f (\bm{W}_t)\|_{\mathrm{F}} \|\bm{W}_{t_2} - \bm{W}_t\|_{\mathrm{F}} \geq f(\bm{W}_t) - \epsilon \approx f(\bm{W}_t)$. 
Hence, Assumption \ref{assum:2}(A4) would not be strong enough to ensure that an optimizer converges.

The following theorem provides a lower convergence rate of mini-batch SGD \eqref{SGD} that converges in the Ces\`{a}ro mean.

\begin{thm}\label{thm:sgd_convergence_rate_lower}
Let $(\bm{W}_t)$ be the sequence generated by mini-batch SGD \eqref{SGD} with $(\eta_t)$ and $(b_t)$ satisfying \eqref{sgd_condition} under Assumptions \ref{assum:1} and \ref{assum:2}. Then, the mean of $(\bm{W}_t)_{t = t_2}^{T + t_2 -1}$ satisfies that, for all $T \in \mathbb{N}$,
\begin{align*}
&\frac{1}{\sum_{t= t_2}^{T + t_2 -1} \eta_t}
\sum_{t= t_2}^{T + t_2 -1} \eta_t 
\mathbb{E} [\| \nabla f (\bm{W}_t)\|_{\mathrm{F}}^2]
= \Omega \left( \frac{1}{\sum_{t= t_2}^{T + t_2 -1} \eta_t} \right) 
\geq
\frac{C_4}{\sum_{t = t_2}^{T + t_2 -1} \eta_t},
\end{align*}
where $t_2 \in \mathbb{N}$ is such that $C_4 \coloneqq \inf \{\mathbb{E}[f(\bm{W}_{t_2})] - \mathbb{E}[f(\bm{W}_t)] \colon t \geq t_2 \} \geq 0$.
\end{thm}

\begin{proof}
\textbf{of Theorem \ref{thm:sgd_convergence_rate_lower}} Assumption \ref{assum:2}(A3) implies that, for all $t \geq t_2 + 1$,
\begin{align*}
\mathbb{E}_{\bm{\xi}_t} \left[ f(\bm{W}_{t+1})|\bm{\xi}_{[t_2 : t-1]} \right]
\geq 
f(\bm{W}_t) 
+ \eta_t 
\underbrace{\mathbb{E}_{\bm{\xi}_t} \left[ \nabla f (\bm{W}_t) \bullet \bm{D}_t^{\mathrm{SGD}} |\bm{\xi}_{[t_2 : t-1]} \right]}_{= - \|\nabla f (\bm{W}_t)\|_{\mathrm{F}}^2 \text{ } \because \text{ } \eqref{eq:1}}.
\end{align*}
Taking the total expectation $\mathbb{E} \coloneqq \mathbb{E}_{\bm{\xi}_{t_2}} \cdots \mathbb{E}_{\bm{\xi}_{t}}$ to the above inequality implies that, for all $t \geq t_2$,
\begin{align*}
\eta_t \mathbb{E}[\|\nabla f (\bm{W}_t)\|_{\mathrm{F}}^2]
\geq 
\mathbb{E}[f(\bm{W}_t)] - \mathbb{E}[f(\bm{W}_{t+1})].
\end{align*}
Let $T \in \mathbb{N}$. Summing the above inequality from $t = t_2$ to $t = T + t_2 -1$ and invoking Assumption \ref{assum:2}(A4) together lead to
\begin{align*}
\frac{1}{\sum_{t = t_2}^{T + t_2 -1} \eta_t} \sum_{t = t_2}^{T + t_2 -1} \eta_t \mathbb{E}[\|\nabla f (\bm{W}_t)\|_{\mathrm{F}}^2]
\geq \frac{\mathbb{E}[f(\bm{W}_{t_2})] - \mathbb{E}[f(\bm{W}_{T + t_2})]}{\sum_{t = t_2}^{T + t_2 -1} \eta_t}
\geq \frac{C_4}{\sum_{t = t_2}^{T + t_2 -1} \eta_t},
\end{align*}
which completes the proof.
\end{proof}

\section{Muon without Momentum: Comparisons with Mini-batch SGD}
\label{sec:4}
The Muon optimizer \citep{jordan2024muon} is updated as follows: Given initial points $\bm{W}_0, \bm{M}_{-1} \in \mathbb{R}^{m \times n}$ and a momentum parameter $\beta \in [0,1)$, 
\begin{align}\label{Muon}
\begin{split}
&\text{[Muon]}\\
&\bm{M}_t = \beta \bm{M}_{t-1} + (1 - \beta) \nabla f_{\bm{\xi}_t} (\bm{W}_t)\\
&\bm{O}_t \in \argmin \{ \| \bm{O} - \bm{M}_t \|_{\mathrm{F}} \colon \bm{O}^\top \bm{O} = \bm{I}_n  \} \\
&\bm{W}_{t+1} 
= \bm{W}_t + \eta_t \bm{D}_t^{\mathrm{Muon}}
= \bm{W}_t - \eta_t \bm{O}_t.
\end{split}
\end{align}

To compare the convergence properties of mini-batch SGD (Section \ref{sec:3}) using the mini-batch gradient \eqref{mini_batch} fairly with those of Muon, we consider the case of a Muon optimizer without momentum, i.e., in the case of $\beta = 0$, minimizing $f$ defined by \eqref{erm} under Assumption \ref{assum:1}:
\begin{align}\label{Muon_without_beta}
\begin{split}
&\text{[Muon with $\beta = 0$]}\\
&\bm{G}_t = \nabla f_{\bm{\xi}_t} (\bm{W}_t)\\
&\bm{O}_t \coloneqq \bm{U}_t \bm{V}_t^\top \in \argmin \{ \| \bm{O} - \bm{G}_t \|_{\mathrm{F}} \colon \bm{O}^\top \bm{O} = \bm{I}_n  \} \\
&\bm{W}_{t+1} 
= \bm{W}_t + \eta_t \bm{D}_t^{\mathrm{Muon}}
= \bm{W}_t - \eta_t \bm{O}_t = \bm{W}_t - \eta_t \bm{U}_t \bm{V}_t^\top,
\end{split}
\end{align}
where $\bm{U}_t \in \mathbb{R}^{m \times r}$ and $\bm{V}_t \in \mathbb{R}^{n \times r}$ are matrices in the singular value decomposition of $\bm{G}_t$, i.e., $\bm{G}_t = \bm{U}_t \bm{\Sigma}_t \bm{V}_t^\top$, and $\bm{\Sigma}_t$ is a diagonal matrix whose diagonal entries are the $r$ singular values of $\bm{G}_t$. An $m \times n$ matrix with orthonormal columns $\bm{O}_t \coloneqq \bm{U}_t \bm{V}_t^\top$ minimizes a function $F_t(\bm{O}) \coloneqq \|\bm{O} - \bm{G}_t \|_{\mathrm{F}}$ over the Stiefel manifold $\mathrm{St}(n,m) \coloneqq \{ \bm{O} \in \mathbb{R}^{m \times n} \colon \bm{O}^\top \bm{O} = \bm{I}_n \}$ \citep[Proposition 4]{bernstein_newhouse_2024_oldoptimizer}. Using $\bm{O}_t \coloneqq \bm{U}_t \bm{V}_t^\top$ is computationally expensive, since it requires the singular value decomposition of $\bm{G}_t$ to be computed. In practice, we use an approximation $\bm{X}_{t,K}$ of $\bm{O}_t \coloneqq \bm{U}_t \bm{V}_t^\top$ that is computed with the following Newton-Schulz iteration $(\bm{X}_{t,k})_{k=0}^K$: 
given $\bm{X}_{t,0} \coloneqq \frac{\bm{G}_t}{\| \bm{G}_t \|_{\mathrm{F}}}$ and $a,b,c \in \mathbb{R}$,
\begin{align*}
\bm{X}_{t, k+1} = a \bm{X}_{t,k} + b \left( \bm{X}_{t,k} \bm{X}_{t,k}^\top \right) \bm{X}_{t,k} + c \left(\bm{X}_{t,k} \bm{X}_{t,k}^\top \right)^2 \bm{X}_{t,k} \to \bm{O}_t \coloneqq \bm{U}_t \bm{V}_t^\top \text{ } (k \to + \infty).
\end{align*}

\subsection{Descent property}
The following lemma gives the descent property of Muon \eqref{Muon_without_beta} with $\beta = 0$ to minimize $f$ defined by \eqref{erm}. 

\begin{lem}\label{lem:muon_without_descent}
{\em 
Let $(\bm{W}_t)$ be the sequence generated by Muon \eqref{Muon_without_beta} with $\beta = 0$ under Assumption \ref{assum:1}, $\bm{\xi}_{[t-1]} \coloneqq \{ \bm{\xi}_0, \cdots, \bm{\xi}_{t-1} \}$, and $L \coloneqq \frac{1}{N} \sum_{i=1}^N L_i$. Under the condition $\nabla f (\bm{W}_t) \neq \bm{O}_{m \times n}$ for all $t \in \{0\} \cup \mathbb{N}$,
\begin{enumerate}
\item[(i)]
$\displaystyle{
\mathbb{E}_{\bm{\xi}_t} \left[\nabla f (\bm{W}_t) \bullet \bm{D}_t^{\mathrm{Muon}}|\bm{\xi}_{[t-1]} \right]\\ 
\leq 
- \| \nabla f (\bm{W}_t) \|_{\mathrm{F}}
+ 
\frac{2^{\frac{2}{\mathfrak{p}}} \sqrt{n} \sigma}{b_t^{\frac{\mathfrak{p} -1}{\mathfrak{p}}}}
}$,
\item[(ii)] 
$\displaystyle{
\mathbb{E}_{\bm{\xi}_t} \left[f (\bm{W}_{t+1}) |\bm{\xi}_{[t-1]} \right] 
<
f(\bm{W}_{t})
+ 
\frac{L n^{\frac{1 + \nu}{2}} \eta_t^{1 + \nu}}{1 + \nu} 
+ 
\frac{2^{\frac{2}{\mathfrak{p}}} \sqrt{n} \sigma \eta_t}{b_t^{\frac{\mathfrak{p} -1}{\mathfrak{p}}}}. 
}$
\end{enumerate}
This implies that, if $(\eta_t^{1 + \nu})$ and $(\eta_t {b_t^{\frac{1 - \mathfrak{p}}{\mathfrak{p}}}})$ converge to $0$, then, for all $\epsilon > 0$, there exists $t_0 \in \mathbb{N}$ such that, for all $t \geq t_0$, $\mathbb{E}_{\bm{\xi}_t} \left[f (\bm{W}_{t+1}) |\bm{\xi}_{[t-1]} \right] < f(\bm{W}_t) + \epsilon$.
}
\end{lem}

Lemma \ref{lem:muon_without_descent}(i) indicates that, if 
$({b_t^{\frac{1 - \mathfrak{p}}{\mathfrak{p}}}})$ converges to $0$ (e.g., $b_t$ increases with each epoch), then the search direction $\bm{D}_t^{\mathrm{Muon}} = - \bm{O}_t = - \bm{U}_t \bm{V}_t^\top$ is a descent direction of $f$ in the sense that 
$\mathbb{E}_{\bm{\xi}_t} [\nabla f (\bm{W}_t) \bullet \bm{D}_t^{\mathrm{Muon}}|\bm{\xi}_{[t-1]} ] 
\leq - \| \nabla f (\bm{W}_t) \|_{\mathrm{F}} + \epsilon
\approx - \| \nabla f (\bm{W}_t) \|_{\mathrm{F}} < 0$.

Let us compare Lemma \ref{lem:sgd_descent}(ii) with Lemma \ref{lem:muon_without_descent}(ii). Lemma \ref{lem:sgd_descent}(ii) shows that, under the conditions $1 + \nu \leq \mathfrak{p}$ and $\eta_t^{\nu} < \frac{2}{L}$, mini-batch SGD with a diminishing step size $\eta_t$ decreases $f$ in the sense that $\mathbb{E}_{\bm{\xi}_t} [f (\bm{W}_{t+1}) |\bm{\xi}_{[t-1]} ] < f(\bm{W}_t) + \epsilon$. We do not know whether conditions $1 + \nu \leq \mathfrak{p}$ (that is used to evaluate $D_t$ in \eqref{comp_d_t} with Jensen's inequality) and $\eta_t^{\nu} < \frac{2}{L}$ (that is used to delete the term $\| \nabla f (\bm{W}_t)\|_{\mathrm{F}}$ in \eqref{ineq_sgd} that comes from \eqref{comp_d_t_1} and Young's inequality) hold before implementing mini-batch SGD, since $\mathfrak{p}$ and $L$ in Assumption \ref{assum:1} are unknown parameters. Hence, we may need to exercise caution when using mini-batch SGD to train DNNs. Meanwhile, Lemma \ref{lem:muon_without_descent}(ii) indicates that Muon \eqref{Muon_without_beta} with $\beta = 0$ and a diminishing step size $\eta_t$ decreases $f$ in the sense that $\mathbb{E}_{\bm{\xi}_t} [f (\bm{W}_{t+1}) |\bm{\xi}_{[t-1]} ] < f(\bm{W}_t) + \epsilon$ without unrealistic conditions, such as $1 + \nu \leq \mathfrak{p}$ and $\eta_t^{\nu} < \frac{2}{L}$. This is because we can evaluate $G_t$ in \eqref{G_T} and $\overline{D}_t$ in \eqref{bar_d_t} without using Jensen's inequality or Young's inequality (see the proof of Lemma \ref{lem:muon_without_descent} for details). 

\begin{proof}
\textbf{of Lemma \ref{lem:muon_without_descent}} 
(i) 
From $\bm{D}_t^{\mathrm{Muon}} \coloneqq - \bm{O}_t = - \bm{U}_t \bm{V}_t^\top$, we have 
\begin{align*} 
G_t 
\coloneqq
\mathbb{E}_{\bm{\xi}_t} \left[\nabla f (\bm{W}_t) \bullet \bm{D}_t^{\mathrm{Muon}}|\bm{\xi}_{[t-1]} \right]
= 
\underbrace{- \mathbb{E}_{\bm{\xi}_t} \left[ \bm{G}_t \bullet \bm{O}_t  \Big| \bm{\xi}_{[t-1]} \right]}_{G_{1,t}}
+ \underbrace{\mathbb{E}_{\bm{\xi}_t} \left[ (\bm{G}_t - \nabla f (\bm{W}_t)) \bullet \bm{O}_t  \Big| \bm{\xi}_{[t-1]} \right]}_{G_{2,t}}.
\end{align*}
From \eqref{Muon_without_beta} and the expansion of the $2$-nd power $\| \bm{O} - \bm{G}_t \|_{\mathrm{F}}^2 = \|\bm{O} \|_{\mathrm{F}}^2 - 2 \bm{G}_t \bullet \bm{O} + \| \bm{G}_t \|_{\mathrm{F}}^2 = - 2 \bm{G}_t \bullet \bm{O} + (n + \| \bm{G}_t \|_{\mathrm{F}}^2)$ ($\bm{O} \in \mathrm{St}(n,m)$), we have 
\begin{align*}
\bm{O}_t  \in \argmin \{ \| \bm{O} - \bm{G}_t \|_{\mathrm{F}}^2 \colon \bm{O}^\top \bm{O} = \bm{I}_n  \}
= \argmax \{  \bm{G}_t \bullet \bm{O} \colon \| \bm{O} \|_2 = 1  \}. 
\end{align*} 
Hence, the definition of the dual norm $\| \cdot \|_{2,*}$ of $\| \cdot \|_{2}$ ensures that $\| \bm{G}_t \|_{2,*} \coloneqq \max \{ \bm{G}_t \bullet \bm{O} \colon \| \bm{O} \|_{2} = 1 \} = \bm{G}_t \bullet \bm{O}_t$. Accordingly, the triangle inequality for $\| \cdot \|_{2,*}$ gives 
\begin{align*}
G_{1,t} 
= - \mathbb{E}_{\bm{\xi}_t} \left[ \| \bm{G}_t \|_{2,*} | \bm{\xi}_{[t-1]} \right]
\leq 
- \| \nabla f (\bm{W}_t) \|_{2,*}
+ \mathbb{E}_{\bm{\xi}_t} \left[ \| \bm{G}_t - \nabla f(\bm{W}_t) \|_{2,*} | \bm{\xi}_{[t-1]} \right].
\end{align*}
This, together with the relation $\| \bm{W} \|_{\mathrm{F}} \leq \|\bm{W} \|_{2,*} \leq \sqrt{n} \|\bm{W}\|_{\mathrm{F}}$ ($\bm{W} \in \mathbb{R}^{m \times n}$) and the same technique used to prove \eqref{p_variance_exp} (i.e., Jensen's inequality) in Example \ref{exp:1}, implies that
\begin{align*}
G_{1,t} 
= - \mathbb{E}_{\bm{\xi}_t} \left[ \| \bm{G}_t \|_{2,*} | \bm{\xi}_{[t-1]} \right]
\leq 
- \| \nabla f (\bm{W}_t) \|_{\mathrm{F}}
+ \sqrt{n} \left( \mathbb{V}_{\bm{\xi}_t}^{\mathfrak{p}} \left[ \nabla f_{\bm{\xi}_t} (\bm{W}_t)  | \bm{\xi}_{[t-1]} \right] \right)^{\frac{1}{\mathfrak{p}}}.
\end{align*}
Proposition \ref{prop:1}(ii) thus ensures that 
\begin{align}\label{g_t_1}
G_{1,t} 
\leq 
- \| \nabla f (\bm{W}_t) \|_{\mathrm{F}}
+ \frac{2^{\frac{2 - \mathfrak{p}}{\mathfrak{p}}} \sqrt{n} \sigma}{b_t^{\frac{\mathfrak{p} -1}{\mathfrak{p}}}}.
\end{align}
From $\bm{O}_t \in \mathrm{St}(n,m)$, we have $\| \bm{O}_t \|_{\mathrm{F}} = \sqrt{n}$. The Cauchy-Schwarz inequality, together with the same technique used to prove \eqref{g_t_1}, ensures that 
\begin{align}\label{g_t_2} 
G_{2,t} \leq \mathbb{E}_{\bm{\xi}_t} \left[ \| \bm{O}_t  \|_{\mathrm{F}} | \bm{\xi}_{[t-1]} \right] \left( \mathbb{V}_{\bm{\xi}_t}^{\mathfrak{p}} \left[ \nabla f_{\bm{\xi}_t} (\bm{W}_t)  | \bm{\xi}_{[t-1]} \right] \right)^{\frac{1}{\mathfrak{p}}}
\leq  
\frac{2^{\frac{2 - \mathfrak{p}}{\mathfrak{p}}} \sqrt{n} \sigma}{b_t^{\frac{\mathfrak{p} -1}{\mathfrak{p}}}}.
\end{align}
From \eqref{g_t_1} and \eqref{g_t_2}, we have 
\begin{align}\label{G_T}
G_t 
\leq 
- \| \nabla f (\bm{W}_t) \|_{\mathrm{F}}
+ 
\frac{2^{\frac{2}{\mathfrak{p}}} \sqrt{n} \sigma}{b_t^{\frac{\mathfrak{p} -1}{\mathfrak{p}}}},
\end{align}
which completes the proof.

(ii) Applying $\bm{W}_1 = \bm{W}_{t+1}$ and $\bm{W}_2 = \bm{W}_t$ to \eqref{g_descent_lemma_sum} and using $\bm{W}_{t+1} - \bm{W}_t = \eta_t \bm{D}_t^{\mathrm{Muon}}$ imply that, for all $t \in \mathbb{N}$,
\begin{align*}
&\mathbb{E}_{\bm{\xi}_t} \left[ f(\bm{W}_{t+1})|\bm{\xi}_{[t-1]} \right]\\
&\leq 
f(\bm{W}_t) 
+ \eta_t 
\underbrace{\mathbb{E}_{\bm{\xi}_t} \left[ \nabla f (\bm{W}_t) \bullet \bm{D}_t^{\mathrm{Muon}} |\bm{\xi}_{[t-1]} \right]}_{G_t}
+ \frac{L \eta_t^{1 + \nu}}{1 + \nu} 
\underbrace{\mathbb{E}_{\bm{\xi}_t} \left[
\left\| \bm{D}_t^{\mathrm{Muon}} \right\|_{\mathrm{F}}^{1 + \nu} \Big| \bm{\xi}_{[t-1]} \right]}_{\overline{D}_t}.
\end{align*}
From the same proof technique \eqref{g_t_2} (i.e., $\|\bm{O}_t \|_{\mathrm{F}} = \sqrt{n}$), we have 
\begin{align}\label{bar_d_t}
\overline{D}_t
= 
\mathbb{E}_{\bm{\xi}_t} \left[
\left\| \bm{O}_t \right\|_{\mathrm{F}}^{1 + \nu} \Big| \bm{\xi}_{[t-1]} \right]
= n^{\frac{1 + \nu}{2}}.
\end{align}
Accordingly, for all $t \in \{0\} \cup \mathbb{N}$,
\begin{align}\label{ineq_muon_without}
\begin{split}
\mathbb{E}_{\bm{\xi}_t} \left[ f(\bm{W}_{t+1})|\bm{\xi}_{[t-1]} \right]
&\leq 
f(\bm{W}_t) 
+ \eta_t 
\left\{ - \| \nabla f (\bm{W}_t) \|_{\mathrm{F}}
+ 
\frac{2^{\frac{2}{\mathfrak{p}}} \sqrt{n} \sigma}{b_t^{\frac{\mathfrak{p} -1}{\mathfrak{p}}}} \right\} 
+ \frac{L \eta_t^{1 + \nu}}{1 + \nu} 
n^{\frac{1 + \nu}{2}}\\
&< 
f(\bm{W}_t)
+
\frac{2^{\frac{2}{\mathfrak{p}}} \sqrt{n} \sigma \eta_t}{b_t^{\frac{\mathfrak{p} -1}{\mathfrak{p}}}}
+ 
\frac{L n^{\frac{1 + \nu}{2}} \eta_t^{1 + \nu}}{1 + \nu}.
\end{split}  
\end{align}
This completes the proof.
\end{proof}

\subsection{Convergence}
The following is a convergence analysis of Muon \eqref{Muon_without_beta} with $\beta = 0$ under Assumption \ref{assum:1}. Theorem \ref{thm:sgd_convergence} indicates that, in order to converge, mini-batch SGD must satisfy the condition $1 + \nu \leq \mathfrak{p}$, while Theorem \ref{thm:muon_without_convergence} indicates that Muon \eqref{Muon_without_beta} requires only the step size $\eta_t$ and batch size $b_t$ to be set. 

\begin{thm}\label{thm:muon_without_convergence}
Let $(\bm{W}_t)$ be the sequence generated by Muon \eqref{Muon_without_beta} with $\beta = 0$ under Assumption \ref{assum:1}. If $(\eta_t)$ and $(b_t)$ satisfy 
\begin{align}\label{muon_without_condition}
\sum_{t=0}^{+\infty} \eta_t = + \infty, \text{ }
\sum_{t=0}^{+\infty} \eta_t^{1 + \nu} < + \infty, \text{ } 
\sum_{t=0}^{+\infty} \frac{\eta_t}{b_t^{\frac{\mathfrak{p} - 1}{\mathfrak{p}}}} < + \infty.
\end{align}
Then, $(\nabla f (\bm{W}_t))$ converges to $\bm{O}_{m \times n}$ almost surely in the sense of the limit inferior. 
\end{thm}

\begin{proof}
\textbf{of Theorem \ref{thm:muon_without_convergence}} Inequality \eqref{ineq_muon_without}, \eqref{muon_without_condition}, and the super martingale convergence theorem \citep[Proposition 8.2.10]{bert}  give
\begin{align*}
\sum_{t=0}^{+ \infty} \eta_t \| \nabla f (\bm{W}_t)\|_{\mathrm{F}} < + \infty \text{ a.s.},
\end{align*}
which, together with $\sum_{t=0}^{+ \infty} \eta_t = + \infty$, implies that $\liminf_{t \to + \infty} \| \nabla f (\bm{W}_t)\|_{\mathrm{F}} = 0$. This completes the proof.
\end{proof}

\subsection{Convergence rate}
\subsubsection{Upper convergence bound}
The following gives an upper convergence rate of Muon \eqref{Muon_without_beta} with $\beta = 0$ that converges in the Ces\`{a}ro mean.

\begin{thm}\label{thm:muon_without_convergence_rate_upper}
Let $(\bm{W}_t)$ be the sequence generated by Muon \eqref{Muon_without_beta} with $\beta = 0$ and $(\eta_t)$ and $(b_t)$ satisfying \eqref{muon_without_condition} under Assumption \ref{assum:1}. Then, the mean of $(\bm{W}_t)_{t = 0}^{T -1}$ satisfies that, for all $T \in \mathbb{N}$,
\begin{align*}
&\frac{1}{\sum_{t= 0}^{T-1} \eta_t}
\sum_{t= 0}^{T -1} \eta_t 
\mathbb{E} [\| \nabla f (\bm{W}_t)\|_{\mathrm{F}}]
= O \left( \frac{1}{\sum_{t= 0}^{T -1} \eta_t} \right) \\
&\leq
\frac{C_1}{\sum_{t= 0}^{T-1} \eta_t}
+ \frac{C_2 (\nu)}{\sum_{t=0}^{T-1} \eta_t} \sum_{t=0}^{+ \infty} \eta_t^{1 + \nu}
+ \frac{C_3 (\mathfrak{p}, \sigma)}{\sum_{t=0}^{T-1} \eta_t} \sum_{t=0}^{+ \infty} \frac{\eta_t}{b_t^{\frac{\mathfrak{p}-1}{\mathfrak{p}}}},
\end{align*}
where $L \coloneqq \frac{1}{N} \sum_{i=1}^N L_i$, $f^\star \in \mathbb{R}$ is such that, for all $\bm{W} \in \mathbb{R}^{m \times n}$, $f(\bm{W}) \geq f^\star$, and 
\begin{align*}
C_1 \coloneqq f(\bm{W}_{0}) - f^{\star}, \text{ }
C_2 (\nu) \coloneqq \frac{L n^{\frac{1 + \nu}{2}} }{1 + \nu}, \text{ } 
C_3 (\mathfrak{p}, \sigma) \coloneqq 2^{\frac{2}{\mathfrak{p}}} \sqrt{n} \sigma.
\end{align*}
\end{thm}

In contrast to mini-batch SGD in Theorem \ref{thm:sgd_convergence_rate_upper} needing the existence of $t_1 \in \mathbb{N}$ such that, for all $t \geq t_1$, $\eta_t^\nu < \frac{2}{L}$, Theorem \ref{thm:muon_without_convergence_rate_upper} says that Muon \eqref{Muon_without_beta} with $\beta = 0$ has a simpler upper convergence bound $O(\frac{1}{\sum_{t=0}^{T-1} \eta_t})$. 

\begin{proof}
\textbf{of Theorem \ref{thm:muon_without_convergence_rate_upper}} Taking the total expectation $\mathbb{E} = \mathbb{E}_t \coloneqq \mathbb{E}_{\bm{\xi}_0} \cdots \mathbb{E}_{\bm{\xi}_t}$ to \eqref{ineq_muon_without} ensures that, for all $t \in \{ 0\} \cup \mathbb{N}$,
\begin{align*}
\eta_t 
\mathbb{E} [\| \nabla f (\bm{W}_t) \|_{\mathrm{F}}]
\leq 
\mathbb{E} [ f(\bm{W}_{t})]
-
\mathbb{E} [ f(\bm{W}_{t+1})]
+
\frac{L n^{\frac{1 + \nu}{2}} \eta_t^{1 + \nu}}{1 + \nu}
+
\frac{2^{\frac{2}{\mathfrak{p}}} \sqrt{n} \sigma \eta_t}{b_t^{\frac{\mathfrak{p} -1}{\mathfrak{p}}}}.
\end{align*}
Summing the above inequality from $t = 0$ to $t = T-1$, where $T \in \mathbb{N}$, implies that
\begin{align*}
\sum_{t=0}^{T-1} \eta_t 
\mathbb{E} [\| \nabla f (\bm{W}_t) \|_{\mathrm{F}}]
\leq 
f(\bm{W}_{0}) - f^\star
+
\frac{L n^{\frac{1 + \nu}{2}}}{1 + \nu}
\sum_{t=0}^{T-1} \eta_t^{1 + \nu}
+
2^{\frac{2}{\mathfrak{p}}} \sqrt{n} \sigma
\sum_{t=0}^{T-1}
\frac{\eta_t}{b_t^{\frac{\mathfrak{p} -1}{\mathfrak{p}}}}.
\end{align*}
The above inequality divided by $\sum_{t=0}^{T-1} \eta_t$ leads to the assertion of Theorem \ref{thm:muon_without_convergence_rate_upper}.
\end{proof}

\subsubsection{Lower convergence bound}
The following presents a lower convergence bound of Muon \eqref{Muon_without_beta} with $\beta = 0$.

\begin{thm}\label{thm:muon_without_convergence_rate_lower}
Let $(\bm{W}_t)$ be the sequence generated by Muon \eqref{Muon_without_beta} with $\beta = 0$ and $(\eta_t)$ and $(b_t)$ satisfying \eqref{muon_without_condition} under Assumptions \ref{assum:1} and \ref{assum:2}. Then, the mean of $(\bm{W}_t)_{t = t_2}^{T + t_2 -1}$ satisfies that, for all $T \in \mathbb{N}$,
\begin{align*}
&\frac{1}{\sum_{t= t_2}^{T + t_2 -1} \eta_t}
\sum_{t= t_2}^{T + t_2 -1} \eta_t 
\mathbb{E} [\| \nabla f (\bm{W}_t)\|_{\mathrm{F}}]
= \Omega \left( \frac{1}{\sum_{t= t_2}^{T + t_2 -1} \eta_t} \right) 
\geq \frac{C_4}{\sum_{t = t_2}^{T + t_2 -1} \eta_t},
\end{align*}
where $t_2 \in \mathbb{N}$ is such that $C_4 \coloneqq \frac{1}{\sqrt{n}} \inf \{\mathbb{E}[f(\bm{W}_{t_2})] - \mathbb{E}[f(\bm{W}_t)] \colon t \geq t_2 \} \geq 0$.
\end{thm}

\begin{proof}
\textbf{of Theorem \ref{thm:muon_without_convergence_rate_lower}} Assumption \ref{assum:2}(A3) implies that, for all $t \in \{0\} \cup \mathbb{N}$,
\begin{align*}
\mathbb{E}_{\bm{\xi}_t} \left[ f(\bm{W}_{t+1})|\bm{\xi}_{[t_2: t-1]} \right]
\geq 
f(\bm{W}_t) 
+ \eta_t 
\underbrace{\mathbb{E}_{\bm{\xi}_t} \left[ \nabla f (\bm{W}_t) \bullet \bm{D}_t^{\mathrm{Muon}} |\bm{\xi}_{[t_2 : t-1]} \right]}_{G_t}.
\end{align*}
The Cauchy-Schwarz inequality, together with $\bm{D}_t^{\mathrm{Muon}} = - \bm{O}_t$ and $\| \bm{O}_t \|_{\mathrm{F}} = \sqrt{n}$, ensures that
\begin{align*}
G_t 
= - \mathbb{E}_{\bm{\xi}_t} \left[ \nabla f (\bm{W}_t) \bullet \bm{O}_t |\bm{\xi}_{[t_2 : t-1]} \right]
\geq  
- \| \nabla f (\bm{W}_t) \|_{\mathrm{F}} \mathbb{E}_{\bm{\xi}_t} \left[ \| \bm{O}_t \|_{\mathrm{F}} |\bm{\xi}_{[t_2 : t-1]} \right]
= - \sqrt{n} \|\nabla f (\bm{W}_t)\|_{\mathrm{F}}, 
\end{align*} 
which implies that, for all $t \geq t_2 + 1$, 
\begin{align*}
\mathbb{E}_{\bm{\xi}_t} \left[ f(\bm{W}_{t+1})|\bm{\xi}_{[t_2 : t-1]} \right]
\geq 
f(\bm{W}_t) 
- \sqrt{n} \eta_t \|\nabla f (\bm{W}_t)\|_{\mathrm{F}}.
\end{align*}
Taking the total expectation $\mathbb{E} \coloneqq \mathbb{E}_{\bm{\xi}_{t_2}} \cdots \mathbb{E}_{\bm{\xi}_{t}}$ to the above inequality implies that, for all $t \geq t_2$,
\begin{align*}
\sqrt{n} \eta_t \mathbb{E}[\|\nabla f (\bm{W}_t)\|_{\mathrm{F}}]
\geq 
\mathbb{E}[f(\bm{W}_t)] - \mathbb{E}[f(\bm{W}_{t+1})].
\end{align*}
Let $T \in \mathbb{N}$. By summing the above inequality from $t = t_2$ to $t = T + t_2 -1$ and invoking Assumption \ref{assum:2}(A4), we have
\begin{align*}
\frac{1}{\sum_{t = t_2}^{T + t_2 -1} \eta_t} \sum_{t = t_2}^{T + t_2 -1} \eta_t \mathbb{E}[\|\nabla f (\bm{W}_t)\|_{\mathrm{F}}]
\geq \frac{\mathbb{E}[f(\bm{W}_{t_2})] - \mathbb{E}[f(\bm{W}_{T + t_2})]}{\sqrt{n} \sum_{t = t_2}^{T + t_2 -1} \eta_t}
\geq \frac{C_4}{\sum_{t = t_2}^{T + t_2 -1} \eta_t},
\end{align*}
which completes the proof.
\end{proof}

\section{Muon}
\label{sec:5}
The singular value decomposition of the matrix $\bm{M}_t = \beta \bm{M}_{t-1} + (1 - \beta) \bm{G}_t$ in Muon \eqref{Muon} is represented by $\bm{M}_t = \bm{U}_t \bm{\Sigma}_t \bm{V}_t^\top$, where $\bm{U}_t \in \mathbb{R}^{m \times r}$, $\bm{V}_t \in \mathbb{R}^{n \times r}$, and $\bm{\Sigma}_t$ is a diagonal matrix whose diagonal entries are the $r$ singular values of $\bm{M}_t$. The matrix $\bm{O}_t \coloneqq \bm{U}_t \bm{V}_t$ minimizes a function $F_t (\bm{O}) \coloneqq \|\bm{O} - \bm{M}_t \|_{\mathrm{F}}$ over $\mathrm{St}(n,m)$. This implies that Muon \eqref{Muon} with $\beta \neq 0$ is structurally almost identical to Muon \eqref{Muon_without_beta}. Hence, we can analyze the convergence of Muon \eqref{Muon} by using the results and proof techniques in Section \ref{sec:4}.

\subsection{Descent property}
The following lemma gives the descent property of Muon \eqref{Muon} to minimize $f$ defined by \eqref{erm}. The only difference from the proof of Lemma \ref{lem:muon_without_descent} is in evaluating $M_{3,t}$ in \eqref{m_3_diff}.

\begin{lem}\label{lem:muon_descent}
{\em 
Let $(\bm{W}_t)$ be the sequence generated by Muon \eqref{Muon} under Assumption \ref{assum:1}, $\bm{\xi}_{[t-1]} \coloneqq \{ \bm{\xi}_0, \cdots, \bm{\xi}_{t-1} \}$, and $L \coloneqq \frac{1}{N} \sum_{i=1}^N L_i$. Under the condition $\nabla f (\bm{W}_t) \neq \bm{O}_{m \times n}$ for all $t \in \{0\} \cup \mathbb{N}$,
\begin{align*}
&\text{(i) }
\mathbb{E}_{\bm{\xi}_t} \left[ \nabla f (\bm{W}_t) \bullet \bm{D}_t^{\mathrm{Muon}} |\bm{\xi}_{[t-1]} \right]\\
&\leq - \| \nabla f (\bm{W}_t) \|_{\mathrm{F}}
+ 2 \sqrt{n}
\left\{  
\beta^t \| \bm{M}_{0} - \nabla f(\bm{W}_{0}) \|_{\mathrm{F}}
+ L n^{\frac{\nu}{2}} \sum_{i=1}^t \beta^i \eta_{t-i}^\nu 
+ (1 - \beta) 2^{\frac{2 - \mathfrak{p}}{\mathfrak{p}}} \sigma
\sum_{i=0}^t 
\frac{\beta^i}{b_{t-i}^{\frac{\mathfrak{p} -1}{\mathfrak{p}}}}
\right\},\\
&\text{(ii) } 
\mathbb{E}_{\bm{\xi}_t} \left[f (\bm{W}_{t+1}) |\bm{\xi}_{[t-1]} \right]
<
f(\bm{W}_t)
+ 
\frac{L n^{\frac{1 + \nu}{2}} \eta_t^{1 + \nu}}{1 + \nu}\\
&\quad + 2 \sqrt{n} \| \bm{M}_{0} - \nabla f(\bm{W}_{0}) \|_{\mathrm{F}} \eta_t \beta^t
+ 2 L n^{\frac{1+\nu}{2}} \eta_t \sum_{i=1}^t \beta^i \eta_{t-i}^\nu
+ 2^{\frac{3 - \mathfrak{p}}{\mathfrak{p}}} (1 - \beta) \sqrt{n} \sigma \eta_t \sum_{i=0}^t 
\frac{\beta^i}{b_{t-i}^{\frac{\mathfrak{p} -1}{\mathfrak{p}}}}.
\end{align*}
This implies that, if $(\eta_t^{1 + \nu})$, $(\eta_t \beta^t)$, $(\eta_t \sum_{i=1}^t \beta^i \eta_{t-i}^\nu)$, and $(\eta_t \sum_{i=0}^t \beta^i b_{t-i}^{\frac{1- \mathfrak{p}}{\mathfrak{p}}})$ converge to $0$, then, for all $\epsilon > 0$, there exists $t_0 \in \mathbb{N}$ such that, for all $t \geq t_0$, $\mathbb{E}_{\bm{\xi}_t} \left[f (\bm{W}_{t+1}) |\bm{\xi}_{[t-1]} \right] < f(\bm{W}_t) + \epsilon$. 
}
\end{lem}

\begin{proof}
\textbf{of Lemma \ref{lem:muon_descent}} 
(i)
From $\bm{D}_t^{\mathrm{Muon}} \coloneqq - \bm{O}_t = - \bm{U}_t \bm{V}_t^\top$, we have 
\begin{align*}
M_t 
\coloneqq
\mathbb{E}_{\bm{\xi}_t} \left[ \nabla f (\bm{W}_t) \bullet \bm{D}_t^{\mathrm{Muon}} |\bm{\xi}_{[t-1]} \right]
= 
\underbrace{- \mathbb{E}_{\bm{\xi}_t} \left[ \bm{M}_t \bullet \bm{O}_t  \Big| \bm{\xi}_{[t-1]} \right]}_{M_{1,t}}
+ \underbrace{\mathbb{E}_{\bm{\xi}_t} \left[ (\bm{M}_t - \nabla f (\bm{W}_t)) \bullet \bm{O}_t  \Big| \bm{\xi}_{[t-1]} \right]}_{M_{2,t}}.
\end{align*}
From \eqref{Muon} and the expansion of the $2$-nd power $\| \bm{O} - \bm{M}_t \|_{\mathrm{F}}^2 = \|\bm{O} \|_{\mathrm{F}}^2 - 2 \bm{M}_t \bullet \bm{O} + \| \bm{M}_t \|_{\mathrm{F}}^2 = - 2 \bm{M}_t \bullet \bm{O} + (n + \| \bm{M}_t \|_{\mathrm{F}}^2)$ ($\bm{O} \in \mathrm{St}(n,m)$), we have 
\begin{align*}
\bm{O}_t  \in \argmin \{ \| \bm{O} - \bm{M}_t \|_{\mathrm{F}}^2 \colon \bm{O}^\top \bm{O} = \bm{I}_n  \}
= \argmax \{  \bm{M}_t \bullet \bm{O} \colon \| \bm{O} \|_2 = 1  \}. 
\end{align*} 
Hence, the definition of the dual norm $\| \cdot \|_{2,*}$ of $\| \cdot \|_{2}$ ensures that $\| \bm{M}_t \|_{2,*} \coloneqq \max \{ \bm{M}_t \bullet \bm{O} \colon \| \bm{O} \|_{2} = 1 \} = \bm{M}_t \bullet \bm{O}_t$. Accordingly, from the triangle inequality for $\| \cdot \|_{2,*}$ and the relation $\| \bm{W} \|_{\mathrm{F}} \leq \|\bm{W} \|_{2,*} \leq \sqrt{n} \|\bm{W}\|_{\mathrm{F}}$ ($\bm{W} \in \mathbb{R}^{m \times n}$), we have
\begin{align}\label{m_3_diff}
\begin{split}
M_{1,t} 
&= - \mathbb{E}_{\bm{\xi}_t} \left[ \| \bm{M}_t \|_{2,*} | \bm{\xi}_{[t-1]} \right]
\leq 
- \| \nabla f (\bm{W}_t) \|_{2,*}
+ \mathbb{E}_{\bm{\xi}_t} \left[ \| \bm{M}_t - \nabla f(\bm{W}_t) \|_{2,*} | \bm{\xi}_{[t-1]} \right]\\
&
\leq 
- \| \nabla f (\bm{W}_t) \|_{\mathrm{F}}
+ \sqrt{n} \underbrace{\mathbb{E}_{\bm{\xi}_t} \left[ \| \bm{M}_t - \nabla f(\bm{W}_t) \|_{\mathrm{F}} | \bm{\xi}_{[t-1]} \right]}_{M_{3,t}}.
\end{split}
\end{align}
Moreover, from the definition of $\bm{M}_t$ and the triangle inequality, we have 
\begin{align*}
M_{3,t} 
&= 
\mathbb{E}_{\bm{\xi}_t} \left[ \| \beta (\bm{M}_{t-1} - \nabla f(\bm{W}_t) )
+ (1 - \beta) (\bm{G}_t -  \nabla f(\bm{W}_t)) \|_{\mathrm{F}}
| \bm{\xi}_{[t-1]} \right]\\
&\leq
\beta \| \bm{M}_{t-1} - \nabla f(\bm{W}_t) \|_{\mathrm{F}}
+ (1 - \beta) \mathbb{E}_{\bm{\xi}_t} \left[ \| \bm{G}_t -  \nabla f(\bm{W}_t) \|_{\mathrm{F}}
| \bm{\xi}_{[t-1]} \right]\\
&\leq
\beta \| \bm{M}_{t-1} - \nabla f(\bm{W}_{t-1}) \|_{\mathrm{F}}
+ \beta \| \nabla f(\bm{W}_{t-1}) - \nabla f(\bm{W}_t)\|_{\mathrm{F}}
+ (1 - \beta) \mathbb{V}_{\bm{\xi}_t}^1 \left[ \nabla f_{\bm{\xi}_t} (\bm{W}_t)
| \bm{\xi}_{[t-1]} \right].
\end{align*}
This, together with Assumption \ref{assum:1}(A1) ($L$-H\"{o}lder smoothness of $f$), $\|\bm{W}_{t-1} - \bm{W}_t \|_{\mathrm{F}} = \eta_{t-1} \| \bm{O}_{t-1} \|_{\mathrm{F}} = \sqrt{n} \eta_{t-1}$, the same technique used to prove \eqref{p_variance_exp} (i.e., Jensen's inequality) in Example \ref{exp:1}, and Proposition \ref{prop:1}(ii), implies that 
\begin{align*}
M_{3,t} 
&\leq
\beta \| \bm{M}_{t-1} - \nabla f(\bm{W}_{t-1}) \|_{\mathrm{F}}
+ \beta L n^{\frac{\nu}{2}} \eta_{t-1}^\nu 
+ (1 - \beta) \left( \mathbb{V}_{\bm{\xi}_t}^{\mathfrak{p}} \left[ \nabla f_{\bm{\xi}_t} (\bm{W}_t)  | \bm{\xi}_{[t-1]} \right] \right)^{\frac{1}{\mathfrak{p}}}\\
&\leq 
\beta \| \bm{M}_{t-1} - \nabla f(\bm{W}_{t-1}) \|_{\mathrm{F}}
+ \beta L n^{\frac{\nu}{2}} \eta_{t-1}^\nu 
+ (1 - \beta) \frac{2^{\frac{2 - \mathfrak{p}}{\mathfrak{p}}} \sigma}{b_t^{\frac{\mathfrak{p} -1}{\mathfrak{p}}}}.
\end{align*}
Induction thus gives
\begin{align*}
M_{3,t} 
&\leq 
\beta \left\{ \beta \| \bm{M}_{t-2} - \nabla f(\bm{W}_{t-2}) \|_{\mathrm{F}}
+ \beta L n^{\frac{\nu}{2}} \eta_{t-2}^\nu 
+ (1 - \beta) \frac{2^{\frac{2 - \mathfrak{p}}{\mathfrak{p}}} \sigma}{b_{t-1}^{\frac{\mathfrak{p} -1}{\mathfrak{p}}}} \right\}
+ \beta L n^{\frac{\nu}{2}} \eta_{t-1}^\nu 
+ (1 - \beta) \frac{2^{\frac{2 - \mathfrak{p}}{\mathfrak{p}}} \sigma}{b_t^{\frac{\mathfrak{p} -1}{\mathfrak{p}}}}\\
&= \beta^2 \| \bm{M}_{t-2} - \nabla f(\bm{W}_{t-2}) \|_{\mathrm{F}} 
+ L n^{\frac{\nu}{2}} \left( \beta^1 \eta_{t-1}^\nu + \beta^2 \eta_{t-2}^\nu \right)
+ (1 - \beta) 2^{\frac{2 - \mathfrak{p}}{\mathfrak{p}}} \sigma 
\left( \frac{\beta^0}{b_{t}^{\frac{\mathfrak{p} -1}{\mathfrak{p}}}} 
+ \frac{\beta^1}{b_{t-1}^{\frac{\mathfrak{p} -1}{\mathfrak{p}}}}
\right) \\
&\leq
\beta^t \| \bm{M}_{0} - \nabla f(\bm{W}_{0}) \|_{\mathrm{F}}
+ L n^{\frac{\nu}{2}} \sum_{i=1}^t \beta^i \eta_{t-i}^\nu 
+ (1 - \beta) 2^{\frac{2 - \mathfrak{p}}{\mathfrak{p}}} \sigma
\sum_{i=0}^t 
\frac{\beta^i}{b_{t-i}^{\frac{\mathfrak{p} -1}{\mathfrak{p}}}}.
\end{align*}
The Cauchy-Schwarz inequality, together with the same technique used to prove \eqref{g_t_1}, ensures that 
\begin{align}\label{m_t_2} 
M_{2,t} \leq \sqrt{n} 
\underbrace{\mathbb{E}_{\bm{\xi}_t} \left[ \| \bm{M}_t - \nabla f(\bm{W}_t) \|_{\mathrm{F}} | \bm{\xi}_{[t-1]} \right]}_{M_{3,t}}.
\end{align}
Therefore, we have 
\begin{align*}
M_t 
&= M_{1,t} + M_{2,t}
\leq 
\left(- \| \nabla f (\bm{W}_t) \|_{\mathrm{F}}
+ \sqrt{n} M_{3,t} \right) 
+ \sqrt{n} M_{3,t}\\
&\leq 
- \| \nabla f (\bm{W}_t) \|_{\mathrm{F}}
+ 2 \sqrt{n}
\left\{  
\beta^t \| \bm{M}_{0} - \nabla f(\bm{W}_{0}) \|_{\mathrm{F}}
+ L n^{\frac{\nu}{2}} \sum_{i=1}^t \beta^i \eta_{t-i}^\nu 
+ (1 - \beta) 2^{\frac{2 - \mathfrak{p}}{\mathfrak{p}}} \sigma
\sum_{i=0}^t 
\frac{\beta^i}{b_{t-i}^{\frac{\mathfrak{p} -1}{\mathfrak{p}}}}
\right\}.
\end{align*}

(ii) Applying $\bm{W}_1 = \bm{W}_{t+1}$ and $\bm{W}_2 = \bm{W}_t$ to \eqref{g_descent_lemma_sum} and using $\bm{W}_{t+1} - \bm{W}_t = \eta_t \bm{D}_t^{\mathrm{Muon}}$ imply that, for all $t \in \mathbb{N}$,
\begin{align*}
&\mathbb{E}_{\bm{\xi}_t} \left[ f(\bm{W}_{t+1})|\bm{\xi}_{[t-1]} \right]\\
&\leq 
f(\bm{W}_t) 
+ \eta_t 
\underbrace{\mathbb{E}_{\bm{\xi}_t} \left[ \nabla f (\bm{W}_t) \bullet \bm{D}_t^{\mathrm{Muon}} |\bm{\xi}_{[t-1]} \right]}_{M_t}
+ \frac{L \eta_t^{1 + \nu}}{1 + \nu} 
\underbrace{\mathbb{E}_{\bm{\xi}_t} \left[
\left\| \bm{D}_t^{\mathrm{Muon}} \right\|_{\mathrm{F}}^{1 + \nu} \Big| \bm{\xi}_{[t-1]} \right]}_{\overline{D}_t}.
\end{align*}
From \eqref{bar_d_t}, we have that $\overline{D}_t = \mathbb{E}_{\bm{\xi}_t} [ \| \bm{O}_t \|_{\mathrm{F}}^{1 + \nu} | \bm{\xi}_{[t-1]}] = n^{\frac{1 + \nu}{2}}$. Accordingly, for all $t \in \{0\} \cup \mathbb{N}$,
\begin{align}\label{ineq_muon}
&\mathbb{E}_{\bm{\xi}_t} \left[ f(\bm{W}_{t+1})|\bm{\xi}_{[t-1]} \right]
\leq 
f(\bm{W}_t) \nonumber\\
&+ \eta_t 
\left[ - \| \nabla f (\bm{W}_t) \|_{\mathrm{F}}
+ 2 \sqrt{n}
\left\{  
\beta^t \| \bm{M}_{0} - \nabla f(\bm{W}_{0}) \|_{\mathrm{F}}
+ L n^{\frac{\nu}{2}} \sum_{i=1}^t \beta^i \eta_{t-i}^\nu 
+ (1 - \beta) 2^{\frac{2 - \mathfrak{p}}{\mathfrak{p}}} \sigma
\sum_{i=0}^t 
\frac{\beta^i}{b_{t-i}^{\frac{\mathfrak{p} -1}{\mathfrak{p}}}}
\right\} \right] \nonumber\\ 
&\quad + \frac{L \eta_t^{1 + \nu}}{1 + \nu} 
n^{\frac{1 + \nu}{2}} \nonumber\\
&< 
f(\bm{W}_t)
+ 2 \sqrt{n} \| \bm{M}_{0} - \nabla f(\bm{W}_{0}) \|_{\mathrm{F}} \eta_t \beta^t
+ 2 L n^{\frac{1+\nu}{2}} \eta_t \sum_{i=1}^t \beta^i \eta_{t-i}^\nu
+ 2^{\frac{3 - \mathfrak{p}}{\mathfrak{p}}} (1 - \beta) \sqrt{n} \sigma \eta_t 
\sum_{i=0}^t 
\frac{\beta^i}{b_{t-i}^{\frac{\mathfrak{p} -1}{\mathfrak{p}}}} \nonumber\\
&\quad+
\frac{L n^{\frac{1 + \nu}{2}} \eta_t^{1 + \nu}}{1 + \nu}.  
\end{align}
This completes the proof.
\end{proof}

\subsection{Convergence}
The following is a convergence analysis of Muon \eqref{Muon} with $\beta \in [0,1)$ under Assumption \ref{assum:1}. We can check that Theorem \ref{thm:muon_convergence} with $\beta = 0$ coincides with Theorem \ref{thm:muon_without_convergence}.

\begin{thm}\label{thm:muon_convergence}
Let $(\bm{W}_t)$ be the sequence generated by Muon \eqref{Muon} with $\beta \in [0,1)$ under Assumption \ref{assum:1}. If $(\eta_t)$ and $(b_t)$ satisfy 
\begin{align}\label{muon_condition}
\begin{split}
&\sum_{t=0}^{+\infty} \eta_t = + \infty, \text{ }
\sum_{t=0}^{+\infty} \eta_t^{1 + \nu} < + \infty, \text{ } 
\sum_{t=0}^{+\infty} \eta_t \beta^t < + \infty,\\
&\sum_{t=0}^{+\infty} \eta_t \sum_{i=1}^t \beta^i \eta_{t-i}^\nu < + \infty, \text{ } 
\sum_{t=0}^{+\infty} \eta_t \sum_{i=0}^t \frac{\beta^i}{b_{t-i}^{\frac{\mathfrak{p} -1}{\mathfrak{p}}}} < + \infty,
\end{split}
\end{align}
then $(\nabla f (\bm{W}_t))$ converges to $\bm{O}_{m \times n}$ almost surely in the sense of the limit inferior. 
\end{thm}

\begin{proof}
\textbf{of Theorem \ref{thm:muon_convergence}} Inequality \eqref{ineq_muon}, \eqref{muon_condition}, and the super martingale convergence theorem \citep[Proposition 8.2.10]{bert}  give 
\begin{align*}
\sum_{t=0}^{+ \infty} \eta_t \| \nabla f (\bm{W}_t)\|_{\mathrm{F}} < + \infty \text{ a.s.},
\end{align*}
which, together with $\sum_{t=0}^{+ \infty} \eta_t = + \infty$, implies that $\liminf_{t \to + \infty} \| \nabla f (\bm{W}_t)\|_{\mathrm{F}} = 0$. This completes the proof.
\end{proof}

\subsection{Convergence Rate}
\subsubsection{Upper convergence bound}
\eqref{ineq_muon} and a discussion similar to the one proving Theorem \ref{thm:muon_without_convergence_rate_upper} lead to an upper bound of Muon \eqref{Muon}.

\begin{thm}\label{thm:muon_convergence_rate_upper}
Let $(\bm{W}_t)$ be the sequence generated by Muon \eqref{Muon} with $\beta \in [0,1)$ and $(\eta_t)$ and $(b_t)$ satisfying \eqref{muon_condition} under Assumption \ref{assum:1}. Then, the mean of $(\bm{W}_t)_{t = 0}^{T -1}$ satisfies that, for all $T \in \mathbb{N}$,
\begin{align*}
&\frac{1}{\sum_{t= 0}^{T-1} \eta_t}
\sum_{t= 0}^{T -1} \eta_t 
\mathbb{E} [\| \nabla f (\bm{W}_t)\|_{\mathrm{F}}]
= O \left( \frac{1}{\sum_{t= 0}^{T -1} \eta_t} \right) \\
&\leq
\frac{C_1}{\sum_{t= 0}^{T-1} \eta_t}
+ \frac{C_2 (\nu)}{\sum_{t=0}^{T-1} \eta_t} \sum_{t=0}^{+ \infty} \eta_t^{1 + \nu}
+ \frac{C_3}{\sum_{t=0}^{T-1} \eta_t} \sum_{t=0}^{+ \infty} \eta_t \beta^t
+ \frac{C_4}{\sum_{t=0}^{T-1} \eta_t} \sum_{t=0}^{+ \infty} \eta_t \sum_{i=1}^t \beta^i \eta_{t-i}^\nu \\
&\quad + \frac{C_5}{\sum_{t=0}^{T-1} \eta_t} \sum_{t=0}^{+ \infty} \eta_t \sum_{i=0}^t 
\frac{\beta^i}{b_{t-i}^{\frac{\mathfrak{p} -1}{\mathfrak{p}}}}
\end{align*}
where $L \coloneqq \frac{1}{N} \sum_{i=1}^N L_i$, $f^\star \in \mathbb{R}$ is such that, for all $\bm{W} \in \mathbb{R}^{m \times n}$, $f(\bm{W}) \geq f^\star$, and 
\begin{align*}
&C_1 \coloneqq f(\bm{W}_{0}) - f^{\star}, \text{ }
C_2 (\nu) \coloneqq \frac{L n^{\frac{1 + \nu}{2}} }{1 + \nu}, \text{ } 
C_3 \coloneqq 2 \sqrt{n} \| \bm{M}_{0} - \nabla f(\bm{W}_{0}) \|_{\mathrm{F}}, \text{ }
C_4 \coloneqq 2 L n^{\frac{1+\nu}{2}},\\
&C_5 (\beta, \mathfrak{p}, \sigma) \coloneqq 2^{\frac{3 - \mathfrak{p}}{\mathfrak{p}}} (1 - \beta) \sqrt{n} \sigma.
\end{align*}
\end{thm}

\begin{proof}
\textbf{of Theorem \ref{thm:muon_convergence_rate_upper}} Taking the total expectation $\mathbb{E} = \mathbb{E}_t \coloneqq \mathbb{E}_{\bm{\xi}_0} \cdots \mathbb{E}_{\bm{\xi}_t}$ to \eqref{ineq_muon} ensures that, for all $t \in \{ 0\} \cup \mathbb{N}$,
\begin{align*}
\eta_t 
\mathbb{E} [\| \nabla f (\bm{W}_t) \|_{\mathrm{F}}]
&\leq 
\mathbb{E} [ f(\bm{W}_{t})]
-
\mathbb{E} [ f(\bm{W}_{t+1})]
+
\frac{L n^{\frac{1 + \nu}{2}} \eta_t^{1 + \nu}}{1 + \nu}
+ 2 \sqrt{n} \| \bm{M}_{0} - \nabla f(\bm{W}_{0}) \|_{\mathrm{F}} \eta_t \beta^t\\
&\quad + 2 L n^{\frac{1+\nu}{2}} \eta_t \sum_{i=1}^t \beta^i \eta_{t-i}^\nu
+ 2^{\frac{3 - \mathfrak{p}}{\mathfrak{p}}} (1 - \beta) \sqrt{n} \sigma \eta_t \sum_{i=0}^t 
\frac{\beta^i}{b_{t-i}^{\frac{\mathfrak{p} -1}{\mathfrak{p}}}}.
\end{align*}
Summing the above inequality from $t = 0$ to $t = T-1$, where $T \in \mathbb{N}$, implies that
\begin{align*}
\sum_{t=0}^{T-1} \eta_t 
\mathbb{E} [\| \nabla f (\bm{W}_t) \|_{\mathrm{F}}]
&\leq 
f(\bm{W}_{0}) - f^\star
+
\frac{L n^{\frac{1 + \nu}{2}}}{1 + \nu}
\sum_{t=0}^{T-1} \eta_t^{1 + \nu}
+ 2 \sqrt{n} \| \bm{M}_{0} - \nabla f(\bm{W}_{0}) \|_{\mathrm{F}} \sum_{t=0}^{T-1} \eta_t \beta^t\\
&\quad + 2 L n^{\frac{1+\nu}{2}} \sum_{t=0}^{T-1} \eta_t \sum_{i=1}^t \beta^i \eta_{t-i}^\nu
+ 2^{\frac{3 - \mathfrak{p}}{\mathfrak{p}}} (1 - \beta) \sqrt{n} \sigma \sum_{t=0}^{T-1} \eta_t \sum_{i=0}^t 
\frac{\beta^i}{b_{t-i}^{\frac{\mathfrak{p} -1}{\mathfrak{p}}}}.
\end{align*}
The above inequality divided by $\sum_{t=0}^{T-1} \eta_t$ leads to the assertion of Theorem \ref{thm:muon_convergence_rate_upper}.
\end{proof}

\subsubsection{Lower convergence bound}
The following gives a lower convergence bound of Muon \eqref{Muon} with $\beta \in [0,1)$. The proof of Theorem \ref{thm:muon_convergence_rate_lower} follows that of Theorem \ref{thm:muon_without_convergence_rate_lower}.

\begin{thm}\label{thm:muon_convergence_rate_lower}
Let $(\bm{W}_t)$ be the sequence generated by Muon \eqref{Muon} with $\beta \in [0,1)$ and $(\eta_t)$ and $(b_t)$ satisfying \eqref{muon_condition} under Assumptions \ref{assum:1} and \ref{assum:2}. Then, the mean of $(\bm{W}_t)_{t = t_2}^{T + t_2 -1}$ satisfies that, for all $T \in \mathbb{N}$,
\begin{align*}
&\frac{1}{\sum_{t= t_2}^{T + t_2 -1} \eta_t}
\sum_{t= t_2}^{T + t_2 -1} \eta_t 
\mathbb{E} [\| \nabla f (\bm{W}_t)\|_{\mathrm{F}}]
= \Omega \left( \frac{1}{\sum_{t= t_2}^{T + t_2 -1} \eta_t} \right) 
\geq \frac{C_6}{\sum_{t = t_2}^{T + t_2 -1} \eta_t},
\end{align*}
where $t_2 \in \mathbb{N}$ is such that $C_6 \coloneqq \frac{1}{\sqrt{n}} \inf \{\mathbb{E}[f(\bm{W}_{t_2})] - \mathbb{E}[f(\bm{W}_t)] \colon t \geq t_2 \} \geq 0$.
\end{thm}

\section{Conclusion}
This paper considered a nonconvex H\"{o}lder-smooth ERM with the boundedness condition of the $\mathfrak{p}$-variance of the stochastic gradient accounting for heavy-tailed stochastic noise. We showed that Muon converges almost surely to appropriate points faster than mini-batch SGD. Our convergence proof indicated that this faster convergence of Muon strongly depends on the search direction using the point on the Stiefel manifold closest to the mini-batch gradient. 



\appendix
\section{Examples of $(\eta_t)$ and $(b_t)$}
\label{app:conditions}
\subsection{Examples of $(\eta_t)$ and $(b_t)$ satisfying \eqref{sgd_condition} and \eqref{muon_without_condition}}
Let $\eta > 0$, $a \in (0,1)$, and $\eta_t \coloneqq \frac{\eta}{(t+1)^a}$ ($t \in \{0\} \cup \mathbb{N}$). We have that, for all $T \in \mathbb{N}$,
\begin{align*}
\sum_{t=0}^{T-1} \eta_t 
\geq
\eta \int_0^T \frac{\mathrm{d}t}{(t+1)^a}
= 
\begin{dcases}
\frac{\eta}{1 - a} \{ (T+1)^{1-a} -1  \} &\text{ } (a \in (0,1))\\
\eta \log (T+1) &\text{ } (a = 1).
\end{dcases}
\end{align*}
We also have 
\begin{align*}
\sum_{t=0}^{T-1} \eta_t^{1 + \nu} 
\leq 
\eta^{1+\nu} \left( 1 + \int_0^{T-1} \frac{\mathrm{d}t}{(t+1)^{(1 + \nu)a}} \right)
\leq
\begin{dcases}
\frac{\eta^{1+\nu}}{1 - (1 + \nu)a} T^{1 - (1 + \nu)a} &\text{ } ((1+ \nu)a < 1 )\\
\eta^{1+\nu} (1 + \log T) &\text{ } ((1+ \nu)a = 1)\\
\frac{(1 + \nu)a \eta^{1 + \nu}}{(1 + \nu)a - 1} &\text{ } (1 < (1+ \nu)a).
\end{dcases}
\end{align*}
Let $b \in \mathbb{N}$, $\delta > 1$, and $b_t \coloneqq b \delta^t$ ($t \in \{0\} \cup \mathbb{N}$). Then, 
\begin{align*}
\sum_{t=0}^{T-1} \frac{\eta_t^{1 + \nu}}{b_t^{\mathfrak{p} - 1}}
\leq 
\frac{\eta^{1+\nu}}{b^{\mathfrak{p} - 1}} \sum_{t=0}^{T-1} \frac{1}{\delta^{(\mathfrak{p} - 1)t}}
\leq 
\frac{\eta^{1+\nu}}{b^{\mathfrak{p} - 1} (\delta^{(\mathfrak{p} - 1)} -1) } 
\end{align*}
and 
\begin{align*}
\sum_{t=0}^{T-1} \frac{\eta_t^{1 + \nu}}{
b_t^{\frac{\mathfrak{p} - 1}{\mathfrak{p}}}}
\leq 
\frac{\eta^{1+\nu}}{b^{\frac{\mathfrak{p} - 1}{\mathfrak{p}}}} \sum_{t=0}^{T-1} \frac{1}{\delta^{\frac{\mathfrak{p} - 1}{\mathfrak{p}} t}}
\leq 
\frac{\eta^{1+\nu}}{b^{\frac{\mathfrak{p} - 1}{\mathfrak{p}}} (\delta^{\frac{\mathfrak{p} - 1}{\mathfrak{p}}} -1) }. 
\end{align*}

\subsection{Examples of $(\eta_t)$ and $(b_t)$ satisfying \eqref{muon_condition}}
Let $\eta_t$ and $b_t$ be the sequences defined as in the above subsection, and $\beta \in [0,1)$. Then,
\begin{align*}
\sum_{t=0}^{T-1} \eta_t \beta^t
\leq \eta \sum_{t=0}^{T-1} \beta^t
\leq \frac{\eta}{1 - \beta}. 
\end{align*}
Moreover,
\begin{align*}
\sum_{t=0}^{T-1} \eta_t \sum_{i=1}^t \beta^i \eta_{t-i}^\nu
\leq
\eta^{1+\nu} \sum_{t=0}^{T-1} \sum_{i=1}^t \beta^i
= 
\eta^{1+\nu}
\sum_{t=0}^{T-1} \frac{\beta (1 - \beta^t)}{1 - \beta}
\leq  
\frac{\eta^{1+\nu}}{1 - \beta}
\end{align*}
and
\begin{align*} 
\sum_{t=0}^{T-1} \eta_t \sum_{i=0}^t \frac{\beta^i}{b_{t-i}^{\frac{\mathfrak{p} -1}{\mathfrak{p}}}} 
\leq 
\frac{\eta}{b^{\frac{\mathfrak{p} - 1}{\mathfrak{p}}}} \sum_{t=0}^{T-1}  \sum_{i=0}^t \frac{\beta^i}{\delta^{\frac{\mathfrak{p} -1}{\mathfrak{p}}(t-i)}}
\leq 
\frac{\eta \delta^{\frac{\mathfrak{p} - 1}{\mathfrak{p}}}}{b^{\frac{\mathfrak{p} - 1}{\mathfrak{p}}}(1 - \beta)(\delta^{\frac{\mathfrak{p} - 1}{\mathfrak{p}}} -1)}.
\end{align*}

\vskip 0.2in
\bibliography{bib}

\end{document}